\documentclass[11pt]{article}

\usepackage{amsmath} 
\usepackage{amsthm} 
\usepackage{amsfonts} 
\usepackage{amssymb} 
\usepackage{stmaryrd} 
\usepackage{mathtools} 
\usepackage{thm-restate} 

\usepackage{xifthen} 

\usepackage{xcolor} 
\usepackage{bm} 
\usepackage{bbm} 

\usepackage{array} 
\usepackage[inline]{enumitem} 

\usepackage{microtype} 
\usepackage{lmodern} 

\usepackage{pdflscape} 
\usepackage{afterpage} 

\usepackage{tikz} 
\usetikzlibrary{cd} 
\usetikzlibrary{calc} 

\usepackage[in]{fullpage} 

\usepackage[hypcap=true]{subcaption} 

\usepackage[%
  plainpages=false,
  pdfpagelabels,
  colorlinks,
  citecolor=black,
  linkcolor=black,
  urlcolor=black,
  filecolor=black,
  bookmarksopen=false
]{hyperref} 
\usepackage[all]{hypcap} 
\hypersetup{hypertexnames=false}

\makeatletter
\newcommand*{\newletterthm@internal}{}
\newcommand*{\newletterthm}[1]{%
  \def\newletterthm@name{#1}%
  \renewcommand*{\newletterthm@internal}[1][]{%
    \ifthenelse{\isempty{##1}}{%
      \expandafter\expandafter\expandafter\newtheorem%
      \expandafter\expandafter\expandafter{%
        \expandafter\newletterthm@name%
        \expandafter}%
      \expandafter{%
        \newletterthm@text}%
      \expandafter\renewcommand%
      \expandafter*%
      \expandafter{%
        \csname the#1\endcsname}{\Alph{#1}}%
    }{%
      \expandafter\expandafter\expandafter\newtheorem%
      \expandafter\expandafter\expandafter{%
        \expandafter\newletterthm@name%
        \expandafter}%
      \expandafter{%
        \newletterthm@text}[##1]%
    \expandafter\renewcommand%
        \expandafter*%
        \expandafter{%
          \csname the#1\endcsname}{\csname the##1\endcsname.\Alph{#1}}%
    }%
  }%
  \newletterthm@newthm%
}

\newcommand*{\newletterthm@newthm}[2][]{%
  \ifthenelse{\isempty{#1}}{%
    \def\newletterthm@text{#2}%
    \newletterthm@internal%
  }{%
    \expandafter\newtheorem\expandafter{\newletterthm@name}[#1]{#2}%
  }%
}
\makeatother

\newtheoremstyle{thmstyle}
  {\medskipamount}
  {\smallskipamount}
  {\slshape}
  {0pt}
  {\bfseries}
  {.}
  { }
  {\thmname{#1}\thmnumber{ #2}{\normalfont\thmnote{ (#3)}}}
  
\newtheoremstyle{plainstyle}
  {\medskipamount}
  {\smallskipamount}
  {\rmfamily}
  {0pt}
  {\bfseries}
  {.}
  { }
  {\thmname{#1}\thmnumber{ #2}{\normalfont\thmnote{ (#3)}}}

\theoremstyle{thmstyle}
\newtheorem{theorem}{Theorem}[section]
\newtheorem{lemma}[theorem]{Lemma}

\newtheorem{question}[theorem]{Question}
\newletterthm{theoremLet}{Theorem}

\theoremstyle{plainstyle}
\newtheorem{definition}[theorem]{Definition}

\newtheorem{remark}[theorem]{Remark}

\makeatletter
\def\refdescformat#1{%
  \phantomsection%
  \let\oldlabel\label%
  \let\label\@gobble%
  \edef\@currentlabel{#1}
  \let\label\oldlabel%
  #1:
}
\makeatother

\newlist{refdesc}{description}{1}
\setlist[refdesc]{format={\refdescformat}}

\newlist{enumdef}{enumerate}{1}
\setlist[enumdef]{before={\leavevmode}, label={\arabic*.}, ref={\thetheorem.\arabic*}}

\setlist[enumerate]{label={\roman*.}, ref={(\roman*)}} 

\makeatletter
\newcommand\notoc@internal[2][]{}
\newcommand{\notoc}[1]{
  \renewcommand{\notoc@internal}[2][]{%
    \let\old@addtocontents\addtocontents%
    \let\addtocontents\@gobbletwo%
    \ifthenelse{\isempty{##1}}{%
      #1{##2}%
    }{%
      #1[##1]{##2}%
    }%
    \let\addtocontents\old@addtocontents%
  }%
  \notoc@internal%
}
\makeatother

\numberwithin{equation}{section} 

\let\epsilon\varepsilon

\newcommand{\rn}{\bm}
\newcommand{\df}{\stackrel{\text{def}}{=}}
\newcommand{\place}{\mathord{-}}
\def\symdiff{\mathbin{\triangle}}
\newcommand{\comp}{\mathbin{\circ}}
\newcommand{\rest}{\mathord{\vert}}
\newcommand{\floor}[1]{\ensuremath{\lfloor#1\rfloor}}
\newcommand{\ceil}[1]{\ensuremath{\lceil#1\rceil}}
\newcommand{\Floor}[1]{\ensuremath{\left\lfloor#1\right\rfloor}}
\newcommand{\Ceil}[1]{\ensuremath{\left\lceil#1\right\rceil}}

\DeclareMathOperator{\rk}{rk}

\DeclareMathOperator{\PAC}{PAC}

\DeclareMathOperator{\HP}{HP}

\DeclareMathOperator{\VC}{VC}
\DeclareMathOperator{\Nat}{Nat}
\DeclareMathOperator{\VCN}{VCN}
\DeclareMathOperator{\DS}{DS}

\DeclareMathOperator{\dom}{dom}
\newcommand{\kpart}[1][k]{#1\operatorname{-part}}

\DeclareMathOperator{\ev}{ev}

\newcommand{\EE}{\mathbb{E}}
\newcommand{\NN}{\mathbb{N}}
\newcommand{\PP}{\mathbb{P}}

\newcommand{\RR}{\mathbb{R}}

\newcommand{\One}{\mathbbm{1}}

\newcommand{\cA}{\mathcal{A}}
\newcommand{\cB}{\mathcal{B}}
\newcommand{\cC}{\mathcal{C}}
\newcommand{\cE}{\mathcal{E}}
\newcommand{\cF}{\mathcal{F}}

\newcommand{\cH}{\mathcal{H}}

\title{%
  A packing lemma for $\VCN_k$-dimension\\
  and learning high-dimensional data
}
\author{%
  Leonardo N.~Coregliano \and%
  Maryanthe Malliaris\thanks{Research partially supported by NSF-BSF 2051825.}%
}
\date{\today}

\begin{document}
\maketitle

\begin{abstract}
  Recently, the authors introduced the theory of high-arity PAC learning, which is well-suited for learning graphs, hypergraphs
  and relational structures. In the same initial work, the authors proved a high-arity analogue of the Fundamental Theorem of
  Statistical Learning that almost completely characterizes all notions of high-arity PAC learning in terms of a combinatorial
  dimension, called the Vapnik--Chervonenkis--Natarajan ($\VCN_k$) $k$-dimension, leaving as an open problem only the
  characterization of non-partite, non-agnostic high-arity PAC learnability.

  In this work, we complete this characterization by proving that non-partite non-agnostic high-arity PAC learnability implies a
  high-arity version of the Haussler packing property, which in turn implies finiteness of $\VCN_k$-dimension. This is done by
  obtaining direct proofs that classic PAC learnability implies classic Haussler packing property, which in turn implies finite
  Natarajan dimension and noticing that these direct proofs nicely lift to high-arity.
\end{abstract}

\section{Introduction}
\label{sec:intro}

Consider the following formulation of the question ``are there few convex sets?'': Given $k\in\NN$ and $\epsilon > 0$, does
there exist $m = m(\epsilon)$ such for every probability measure $\mu$ over $\RR$, there exist $m$ convex sets
$A_1,\ldots,A_m\subseteq\RR^k$ such that every convex set $B\subseteq\RR^k$ is $\epsilon$-close to one of the $A_i$ in the
natural sense of the product measure $\mu^k$, i.e., $\mu^k(A_i\symdiff B) < \epsilon$ (note that $m$ depends only on $\epsilon$
and not on $\mu$)? More generally, we could ask:
\begin{question}\label{qst:compression}
  Which classes $\cH$ of subsets of $\RR^k$ admit such a ``compression property''?
\end{question}

The main result of this paper provides a complete answer to the general question in terms of a new notion of combinatorial
dimension (in particular, the convex sets do have this ``compression property''). To continue to illustrate this using the
example of convex sets, the reader familiar with the Haussler packing property might want to first consider the case $k=1$, in
which the class of convex sets amounts to the class of intervals; this class has finite VC~dimension and the result above is
exactly the Haussler packing property. However, when $k\geq 2$, the class of convex sets no longer has finite VC~dimension (even
in the plane, notice that $n$ points along a circle are easily shattered by convex sets), so usual Haussler theory provably does
not apply. In the present work, we will indeed use a general combinatorial dimension ($\VCN_k$-dimension) which specializes to
VC~dimension in the $k=1$ case, and we will obtain a characterization of this ``compression property'', which we call high-arity
Haussler packing property, in terms of finiteness of this dimension. In order to explain what is new and non-trivial about this
result, let us explain some ingredients of the argument.

The proof builds on a recent breakthrough technology of high-arity PAC learning~\cite{CM24+}. We explain this formally below, but
briefly, this theory allows for statistical learning to happen in much more complex settings than classical PAC theory. The key
ingredient is understanding how to leverage structured-correlation in high-dimensional data to make new kinds of learning
possible.

In fact, the results of the present paper answer a major open question of the high-arity PAC paper, by closing an equivalence
of the high-arity PAC theory as we will explain below.

\section{Technical preliminaries}

In the PAC learning theory of Valiant~\cite{Val84} (see also~\cite{SB14} for a more thorough and modern introduction to the
topic), an adversary picks a function $F\colon X\to Y$ out from some family $\cH\subseteq Y^X$ and a probability measure $\mu$
over $X$ and we are tasked to learn $F$ from i.i.d.\ samples of the form $(\rn{x}_i,F(\rn{x}_i))_{i=1}^m$, where each $\rn{x}_i$
is drawn from $\mu$; our answer is required to be probably approximately correct (PAC) in the sense of having small total loss
with high probability over the sample $\rn{x}$. The Fundamental Theorem of Statistical Learning characterizes PAC learnability
of a family $\cH$ in terms of finiteness of:
\begin{itemize}
\item its Vapnik--Chervonenkis ($\VC$) dimension~\cite{VC71,BEH89,VC15} when $\lvert Y\rvert=2$,
\item its Natarajan dimension~\cite{Nat89} when $Y$ is finite,
\item its Daniely--Shalev-Shwartz ($\DS$) dimension~\cite{DS14,BCDMY22} in the general case.
\end{itemize}
See Theorem~\ref{thm:FTSL} below for the case when $Y$ is finite.

In fact, Haussler~\cite{Hau92} showed (when $Y$ is finite, but the result was later extended to the general
setting~\cite{BCDMY22}) that the above is also equivalent to agnostic PAC learnability of $\cH$, that is, even if we allow our
adversary to pick instead a probability measure $\nu$ over $X\times Y$ and provide us i.i.d.\ samples from $\nu$, then we can at
least be competitive in the sense that our answer will have total loss close to the best performing member of $\cH$ (with high
probability over the sample).

In a different work, Haussler~\cite[Corollary~1, Theorem~2]{Hau95} also showed that finiteness of the Natarajan dimension of a
family $\cH\subseteq Y^X$ (with $Y$ finite) is equivalent to the following packing property\footnote{This is a slight
generalization and reformulation of Haussler's original work. Instead, Haussler's original work frames this in terms of
$\epsilon$-separated sets and considers only the case $Y=\{0,1\}$ (which yields a characterization in terms of the
$\VC$-dimension).}: for every $\epsilon > 0$, there exists $m=m(\epsilon)$ such that for every probability measure $\mu$ over
$X$, there exist $\cH'\subseteq\cH$ with $\lvert\cH'\rvert\leq m$ such that every $F\in\cH$ is $\epsilon$-close to some
$H\in\cH'$ in the sense that
\begin{equation*}
  \mu(\{x\in X\mid F(x)\neq H(x)\}) \leq \epsilon.
\end{equation*}

Collectively, these works yield the Fundamental Theorem of Statistical Learning:
\begin{theoremLet}[\cite{VC71,BEH89,Nat89,Hau92,Hau95,VC15}]\label{thm:FTSL}
  The following are equivalent for a family $\cH\subseteq Y^X$ of functions $X\to Y$ with $Y$ finite:
  \begin{enumerate}
  \item The Natarajan dimension of $\cH$ is finite.
  \item $\cH$ has the uniform convergence property.
  \item $\cH$ is agnostically PAC learnable.
  \item $\cH$ is PAC learnable.
  \item $\cH$ has the Haussler packing property.
  \end{enumerate}
\end{theoremLet}

One of the main criticisms of classic PAC learning theory is that it strongly relies on the independence of its samples. While
there has been considerable work~\cite{HL94,AV95,Gam03,ZZX09,SW10,ZXC12,ZLX12,ZXX14,BGS18,SS23} in extending PAC theory to allow
for some correlation, these works see correlation as an obstacle to learning. More recently, the authors introduced the theory
of high-arity PAC learning~\cite{CM24+}, in which correlation is leveraged to increase the learning power.

High-arity PAC learning theory is heavily inspired by the problem of learning a graph over a set $X$ (but more generally, also
covers hypergraphs and even relational structures). By encoding such a graph by its adjacency matrix $F\colon X\times
X\to\{0,1\}$, one could simply apply classic PAC learning theory; however, this yields a rather unnatural learning framework:
the adversary is picking a measure $\mu$ over $X\times X$ and revealing several pairs $(\rn{x}_i,\rn{x'}_i)$ drawn i.i.d.\ from
$\mu$ along with their adjacency $F(\rn{x}_i,\rn{x'}_i)$.

Instead, the setup of $2$-PAC learning is much more natural: the adversary picks a measure $\mu$ over $X$, draws $m$ vertices
$(\rn{x}_i)_{i=1}^m$ i.i.d.\ from $\mu$ and reveals all the adjacency information between these vertices:
$(F(\rn{x}_i,\rn{x}_j))_{i,j=1}^m$ (see Section~\ref{sec:simpleold} for a formal, but simplified definition and see
Appendix~\ref{sec:higharityPACdefs} for the full definitions).

In high-arity, there is also a natural partite framework. Namely, if we were trying to learn a bipartite graph $F\colon
X_1\times X_2\to\{0,1\}$ with a known bipartition $(X_1,X_2)$, we could instead allow the adversary to pick different measures
$\mu_1$ and $\mu_2$ over $X_1$ and $X_2$, respectively, draw $(\rn{x}^1_i)_{i=1}^m$ i.i.d.\ from $\mu_1$ and
$(\rn{x}^2_i)_{i=1}^m$ i.i.d.\ from $\mu_2$, independently from the previous points and provide us all adjacency information:
$(F(\rn{x}^1_i,\rn{x}^2_j))_{i,j=1}^m$.

Every (not necessarily bipartite) graph $F\colon X\times X\to\{0,1\}$ can be interpreted as a bipartite graph
$F^{\kpart[2]}\colon X_1\times X_2\to\{0,1\}$ in which $X_1=X_2=X$ (combinatorially, this is doubling the vertices of $F$), but
a priori $2$-PAC learnability of a family $\cH\subseteq\{0,1\}^{X\times X}$ is not necessarily the same as partite $2$-PAC
learnability of its partization $\cH^{\kpart[2]}\df\{F^{\kpart[2]} \mid F\in\cH\}$.

The interplay between a high-arity hypothesis class $\cH$ and its partization $\cH^{\kpart}$ plays a crucial role in proving
Theorem~\ref{thm:highFTSL} below, which is high-arity analogue of Theorem~\ref{thm:FTSL} (we direct the reader again to
Section~\ref{sec:simpleold} and Appendix~\ref{sec:higharityPACdefs} or to the original paper~\cite{CM24+} for the precise
definitions of the concepts in this theorem). To illustrate the non-triviality of the interplay between partite and non-partite,
consider the partite version of Question~\ref{qst:compression} in Section~\ref{sec:intro}: if we change the setup to allow for
product probability measures of the form $\mu_1\otimes\cdots\otimes\mu_k$, of not necessarily the same measure, why should we
expect the resulting property to be equivalent to the one only for measures of the form $\mu^k$?

\begin{theoremLet}[\protect{\cite[Theorems~1.1 and~5.1]{CM24+}}]\label{thm:highFTSL}
  Let $k\in\NN_+$. The following are equivalent for a family $\cH\subseteq Y^{X^k}$ of functions $X^k\to Y$ with $Y$ finite and
  its partization $\cH^{\kpart}$:
  \begin{enumerate}
  \item The Vapnik--Chervonenkis--Natarajan $k$-dimension of $\cH$ is finite.
  \item The partite Vapnik--Chervonenkis--Natarajan $k$-dimension of $\cH^{\kpart}$ is finite.
  \item $\cH^{\kpart}$ has the uniform convergence property.
  \item $\cH$ is agnostically $k$-PAC learnable.
  \item $\cH^{\kpart}$ is partite agnostically $k$-PAC learnable.
  \item $\cH^{\kpart}$ is partite $k$-PAC learnable.
  \end{enumerate}
  Furthermore, any of the items above implies the following:
  \begin{enumerate}[resume]
  \item\label{it:PAC} $\cH$ is $k$-PAC learnable.
  \end{enumerate}
\end{theoremLet}
(The statement above is the informal version~\cite[Theorem~1.1]{CM24+} that covers only the $0/1$-loss function; the formal
version~\cite[Theorem~5.1]{CM24+} covers general loss functions (under mild assumptions).)

It turns out that the implication of PAC learnability to finite Natarajan dimension (usually called No Free Lunch Theorem) of
classic PAC theory only lifts naturally to the partite setting. The authors were able to partially solve this issue by providing
a ``departization'' operation that allows non-partite agnostic $k$-PAC learnability to imply its partite analogue. However, this
left open the question of whether non-partite $k$-PAC learnability (item~\ref{it:PAC} in Theorem~\ref{thm:highFTSL}) can also be
included in the equivalence list.

In this paper, we prove that this is indeed the case by developing a high-arity version of the Haussler packing property (which
we refer to here as $k$-ary Haussler packing). Recall, the classic equivalence of the Haussler packing property with finiteness
of $\VC$-dimension can be proved by direct implications between the properties\footnote{In fact, this essentially goes back to
Haussler~\cite{Hau95}; however, he does not prove that the packing property implies finite $\VC$-dimension, he instead shows
tightness of an analogous bound when the domain $X$ is finite; but one can derive the implication from his results.}. In the
present work, we find high-arity proofs of the implications $k$-PAC learnability $\implies$ $k$-ary Haussler packing property
and $k$-ary Haussler packing property $\implies$ finite $\VCN_k$-dimension (Theorem~\ref{thm:main} below). The reader should
note that when specialized to $k=1$, these yield the expected direct implications from classic PAC to Haussler packing and from
Haussler packing to finite Natarajan dimension.

\begin{theoremLet}[informal version of Theorems~\ref{thm:simple:PAC->HP} and~\ref{thm:simple:HP->VCNk} of the present paper]%
  \label{thm:main}
  Let $k\in\NN_+$. Then the following hold for a family $\cH\subseteq Y^{X^k}$ of functions $X^k\to Y$ with $Y$ finite:
  \begin{enumerate}
  \item If $\cH$ is $k$-PAC learnable, then $\cH$ has the $k$-ary Haussler packing property.
  \item If $\cH$ has the $k$-ary Haussler packing property, then $\cH$ has finite $\VCN_k$-dimension.
  \end{enumerate}
\end{theoremLet}
Note, at the risk of stating the obvious, that two things have happened here: there is a new high-arity notion defined, that of
$k$-ary Haussler packing, and two new direct implications that close the loop of equivalences when put together with the earlier
Theorem~\ref{thm:highFTSL}. Thus an immediate consequence of the main result in this paper is the following summary theorem (see
Figure~\ref{fig:roadmap} for a pictorial view of the implications in this work and of~\cite{CM24+}):
\begin{theoremLet}\label{thm:fullhighFTSL}
  Let $k\in\NN_+$. The following are equivalent for a family $\cH\subseteq Y^{X^k}$ of functions $X^k\to Y$ with $Y$ finite and
  its partization $\cH^{\kpart}$:
  \begin{enumerate}
  \item\label{thm:fullhighFTSL:VCNk} The Vapnik--Chervonenkis--Natarajan $k$-dimension of $\cH$ is finite.
  \item The partite Vapnik--Chervonenkis--Natarajan $k$-dimension of $\cH^{\kpart}$ is finite.
  \item $\cH^{\kpart}$ has the uniform convergence property.
  \item $\cH$ is agnostically $k$-PAC learnable.
  \item $\cH^{\kpart}$ is partite agnostically $k$-PAC learnable.
  \item $\cH$ is $k$-PAC learnable.
  \item $\cH^{\kpart}$ is partite $k$-PAC learnable.
  \item\label{thm:fullhighFTSL:HP} $\cH$ has the $k$-ary Haussler packing property.
  \item $\cH^{\kpart}$ has the $k$-ary Haussler packing property.
  \end{enumerate}
\end{theoremLet}

In particular, the equivalence of items~\ref{thm:fullhighFTSL:VCNk} and~\ref{thm:fullhighFTSL:HP} above completely answer
Question~\ref{qst:compression}.

\medskip

The paper is organized as follows. First, since our main goal is to complete the characterization of (non-partite) $k$-PAC
learnability, we opt to provide in the main text only a simplified version of the high-arity PAC definitions in~\cite{CM24+} and
of our argument that only covers ``rank at most $1$'' hypotheses in the non-partite setting, ignoring the subtlety of
``high-order variables'' (as these are mostly used in agnostic high-arity PAC, which is not needed for our results); we provide
the full version of the high-arity definitions and of our argument in the appendices. Sections~\ref{sec:simpleold},
\ref{sec:simplenew} and~\ref{sec:main} contain the simplified versions of the definitions from~\cite{CM24+}, of our new
definitions and of our main results, respectively. Appendix~\ref{sec:higharityPACdefs} contains the full version of the same
content and also covers a partite version of the Haussler packing property. Appendix~\ref{sec:partization} relates the
non-partite and partite Haussler packing properties directly through the partization operation.

\section{Simplified high-arity PAC definitions}
\label{sec:simpleold}

In this section, we lay out the definitions of high-arity PAC theory of~\cite{CM24+} that we will need in a simplified manner
that covers only rank at most $1$ hypotheses. We direct the curious reader to Appendix~\ref{sec:higharityPACdefs} for the
general version of the same definitions.

\begin{definition}[\protect{\cite[\S 3]{CM24+}} simplified]
  By a Borel space, we mean a standard Borel space, i.e., a measurable space that is Borel-isomorphic to a Polish space when
  equipped with the $\sigma$-algebra of Borel sets. The space of probability measures on a Borel space $\Lambda$ is denoted
  $\Pr(\Lambda)$.

  Let $\Omega=(X,\cB)$ and $\Lambda=(Y,\cB')$ be non-empty Borel spaces and $k\in\NN_+$.
  \begin{enumdef}
  \item The set of \emph{$k$-ary hypotheses} from $\Omega$ to $\Lambda$, denoted $\cF_k(\Omega,\Lambda)$, is the set of (Borel)
    measurable functions from $\Omega^k$ to $\Lambda$.
  \item A \emph{$k$-ary hypothesis class} is a subset $\cH$ of $\cF_k(\Omega,\Lambda)$ equipped with a $\sigma$-algebra such
    that
    \begin{enumerate}
    \item the evaluation map $\ev\colon\cH\times\Omega^k\to\Lambda$ given by $\ev(H,x)\df H(x)$ is measurable;
    \item for every $H\in\cH$, the set $\{H\}$ is measurable;
    \item for every Borel space $\Upsilon$ and every measurable set $A\subseteq\cH\times\Upsilon$, the projection of $A$ onto
      $\Upsilon$, i.e., the set
      \begin{equation*}
        \{\upsilon\in\Upsilon \mid \exists H\in\cH, (H,\upsilon)\in A\},
      \end{equation*}
      is universally measurable\footnote{\label{ftn:univmeas}This assumption about hypothesis classes is not made
      in~\cite{CM24+}, but it is clearly needed for uniform convergence to make sense; in this document, we will not need
      this. Also, note that if $\cH$ is equipped with a $\sigma$-algebra that turns it into a standard Borel space, then this
      hypothesis is immediately satisfied as Suslin sets are universally measurable.} (i.e., measurable in every completion of
      a probability measure on $\Upsilon$).
    \end{enumerate}
  \item A \emph{$k$-ary loss function} over $\Lambda$ is a measurable function
    $\ell\colon\Omega^k\times\Lambda^{S_k}\times\Lambda^{S_k}\to\RR_{\geq 0}$, where $S_k$ is the symmetric group on
    $[k]\df\{1,\ldots,k\}$. We further define
    \begin{align*}
      \lVert\ell\rVert_\infty & \df \sup_{\substack{x\in\Omega^k\\y,y'\in\Lambda^{S_k}}} \ell(x,y,y'),
      &
      s(\ell) & \df \inf_{\substack{x\in\Omega^k\\y,y'\in\Lambda^{S_k}\\y\neq y'}} \ell(x,y,y'),
    \end{align*}
    and we say that $\ell$ is:
    \begin{description}[format={\normalfont\textit}]
    \item[bounded] if $\lVert\ell\rVert_\infty < \infty$.
    \item[separated] if $s(\ell) > 0$ and $\ell(x,y,y)=0$ for every $x\in\Omega^k$ and every $y\in\Lambda^{S_k}$.
    \end{description}

    If we are further given $k$-ary hypotheses $F,H\in\cF_k(\Omega,\Lambda)$ and a probability measure $\mu\in\Pr(\Omega)$, then
    we define the \emph{total loss} of $H$ with respect to $\mu$, $F$ and $\ell$ by
    \begin{equation*}
      L_{\mu,F,\ell}(H)
      \df
      \EE_{\rn{x}\sim\mu^k}\biggl[
        \ell\Bigl(
        \rn{x},
        \bigl(H(\rn{x}_{\sigma(1)},\ldots,\rn{x}_{\sigma(k)})\bigr)_{\sigma\in S_k},
        \bigl(F(\rn{x}_{\sigma(1)},\ldots,\rn{x}_{\sigma(k)})\bigr)_{\sigma\in S_k}
        \Bigr)
        \biggr].
    \end{equation*}
  \item We say that $F\in\cF_k(\Omega,\Lambda)$ is \emph{realizable} in $\cH\subseteq\cF_k(\Omega,\Lambda)$ with respect to a
    $k$-ary loss function $\ell$ and $\mu\in\Pr(\Omega)$ if $\inf_{H\in\cH} L_{\mu,F,\ell}(H) = 0$.
  \item The \emph{$k$-ary $0/1$-loss function} over $\Lambda$ is defined as $\ell_{0/1}(x,y,y')\df\One[y\neq y']$.
  \item A \emph{($k$-ary) learning algorithm} for a $k$-ary hypothesis class $\cH$ is a measurable function
    \begin{equation*}
      \cA\colon\bigcup_{m\in\NN} \bigl(\Omega^m\times\Lambda^{([m])_k}\bigr)\to\cH,
    \end{equation*}
    where $([m])_k$ denotes the set of injections $[k]\to[m]$.
  \item We say that a $k$-ary hypothesis class is \emph{$k$-PAC learnable} with respect to a $k$-ary loss function $\ell$ if
    there exists a learning algorithm $\cA$ for $\cH$ and a function $m^{\PAC}_{\cH,\ell,\cA}\colon(0,1)^2\to\RR_{\geq 0}$ such
    that for every $\epsilon,\delta\in(0,1)$, every $\mu\in\Pr(\Omega)$ and every $F\in\cF_k(\Omega,\Lambda)$ that is realizable
    in $\cH$ with respect to $\ell$ and $\mu$, we have
    \begin{equation*}
      \PP_{\rn{x}\sim\mu^m}\Biggl[
        L_{\mu,F,\ell}\biggl(
        \cA\Bigl(
        \rn{x},
        \bigl(F(\rn{x}_{\alpha(1)},\ldots,\rn{x}_{\alpha(k)})\bigr)_{\alpha\in ([m])_k}
        \Bigr)
        \biggr)
        \leq \epsilon
        \Biggr] \geq 1 - \delta
    \end{equation*}
    for every integer $m\geq m^{\PAC}_{\cH,\ell,\cA}(\epsilon,\delta)$. A learning algorithm $\cA$ satisfying the above is
    called a \emph{$k$-PAC learner} for $\cH$ with respect to $\ell$.
  \item For a $k$-ary hypothesis $H\in\cF_k(\Omega,\Lambda)$ and $x\in\Omega^{k-1}$, we let $H_x\colon\Omega\to\Lambda^{S_k}$ be
    defined by
    \begin{equation*}
      H_x(x_k)_\sigma \df H(x_{\sigma(1)},\ldots,x_{\sigma(k)}) \qquad (x_k\in\Omega, \sigma\in S_k).
    \end{equation*}

    For a $k$-ary hypothesis class $\cH$, the \emph{Vapnik--Chervonenkis--Natarajan $k$-dimension} of $\cH$
    (\emph{$\VCN_k$-dimension}) is defined as
    \begin{equation*}
      \VCN_k(\cH)\df\sup_{x\in\Omega^{k-1}} \Nat\bigl(\cH(x)\bigr),
    \end{equation*}
    where
    \begin{equation*}
      \cH(x) \df \{H_x \mid H\in\cH\}
    \end{equation*}
    and $\Nat$ is the Natarajan-dimension (see Definition~\ref{def:Nat} below).
  \end{enumdef}
\end{definition}

\begin{definition}[Natarajan dimension~\cite{Nat89}]\label{def:Nat}
  Let $\cF$ be a collection of functions of the form $X\to Y$ and let $A\subseteq X$.
  \begin{enumdef}
  \item We say that $\cF$ \emph{(Natarajan-)shatters} $A$ if there exist functions $f_0,f_1\colon A\to Y$ such that
    \begin{enumerate}
    \item for every $a\in A$, we have $f_0(a)\neq f_1(a)$,
    \item for every $U\subseteq A$, there exists $F_U\in\cF$ such that for every $a\in A$ we have
      \begin{equation*}
        F_U(a) = f_{\One[a\in U]}(a) = 
        \begin{dcases*}
          f_0(a), & if $a\notin U$,\\
          f_1(a), & if $a\in U$.
        \end{dcases*}
      \end{equation*}
    \end{enumerate}
  \item The \emph{Natarajan dimension} of $\cF$ is defined as
    \begin{equation*}
      \Nat(\cF)\df\sup\{\lvert A\rvert \mid A\subseteq X\land\cF\text{ Natarajan-shatters } A\}.
    \end{equation*}
  \end{enumdef}
\end{definition}

\section{Simplified versions of new high-arity concepts}
\label{sec:simplenew}

In this section, we formalize the high-arity version of the Haussler packing property in the same simplified manner as in
Section~\ref{sec:simpleold} and the notion of a metric loss (also in simplified manner); again, we direct the curious reader to
Appendix~\ref{subsec:newhigharityPACdefs} for the general version of the same definitions.

\begin{definition}[$k$-ary Haussler packing property]\label{def:HP}
  Let $k\in\NN_+$, let $\Omega$ and $\Lambda$ be non-empty Borel spaces, let $\cH\subseteq\cF_k(\Omega,\Lambda)$ be a $k$-ary
  hypothesis class and let $\ell\colon\Omega^k\times\Lambda^{S_k}\times\Lambda^{S_k}\to\RR_{\geq 0}$ be a $k$-ary loss function.

  We say that $\cH$ has the \emph{$k$-ary Haussler packing property}\footnote{It would be perhaps more fitting to call this the
  Haussler covering property, but we opt retain the historical name of its unary version.} with respect to $\ell$ if there
  exists a function $m^{\HP}_{\cH,\ell}\colon(0,1)\to\RR_{\geq 0}$ such that for every $\epsilon\in(0,1)$ and every
  $\mu\in\Pr(\Omega)$, there exists $\cH'\subseteq\cH$ with $\lvert\cH'\rvert\leq m^{\HP}_{\cH,\ell}(\epsilon)$ such that for
  every $F\in\cH$, there exists $H\in\cH'$ such that $L_{\mu,F,\ell}(H)\leq\epsilon$. We refer to elements of $\cH'$ as
  \emph{$k$-ary Haussler centers} of $\cH$ at precision $\epsilon$ with respect to $\mu$ and $\ell$.
\end{definition}

We point out that the usual Haussler packing property does a priori make sense in the high-arity setting but would be relative
to a probability measure $\nu$ on $\Omega^k$ and hence would be applicable to very special classes. By contrast, $k$-ary
Haussler considers product measures $\mu^k$ for a probability measure $\mu$ over $\Omega$. Recall that our goal is to fully
characterize which hypothesis classes admit such a property. We will see in the course of the proofs the equivalence of $k$-ary
Haussler packing to $k$-ary PAC learning requires interesting mathematical tools. For simplicity, in the remainder of this
paper, for readability we may drop ``$k$-ary'' from the terminology.

\begin{definition}[Metric loss functions]
  Let $k\in\NN_+$ and let $\Omega$ and $\Lambda$ be non-empty Borel spaces.

  We say that a $k$-ary loss function $\ell\colon\Omega^k\times\Lambda^{S_k}\times\Lambda^{S_k}\to\RR_{\geq 0}$ is \emph{metric}
  if for every $x\in\Omega^k$, the function $\ell(x,\place,\place)$ is a metric on $\Lambda^{S_k}$ in the usual sense, that is,
  the following hold for every $x\in\Omega^k$ and every $y,y',y''\in\Lambda^{S_k}$:
  \begin{enumerate}
  \item We have $\ell(x,y,y')=\ell(x,y',y)$.
  \item We have $\ell(x,y,y')=0$ if and only if $y = y'$.
  \item We have $\ell(x,y,y'')\leq \ell(x,y,y') + \ell(x,y',y'')$.
  \end{enumerate}
\end{definition}

\section{Main results}
\label{sec:main}

In this section we prove that $k$-PAC learnability implies the Haussler packing property (Theorem~\ref{thm:simple:PAC->HP}),
which in turn implies finite $\VCN_k$-dimension (Theorem~\ref{thm:simple:HP->VCNk}). The former theorem is done under the
assumption that the loss function $\ell$ is metric, but we point out that one could have assumed instead that $\ell$ is
separated and bounded with a slight change in the argument (this is featured in the non-simplified version of the implication,
Theorem~\ref{thm:PAC->HP}). Finally, let us point out that the $0/1$-loss function $\ell_{0/1}$ satisfies all hypotheses of
Theorems~\ref{thm:simple:PAC->HP} and~\ref{thm:simple:HP->VCNk} and~\cite[Theorems~1.1 and~5.1]{CM24+}.

\begin{theorem}[$k$-PAC learnability implies Haussler packing property, simplified]\label{thm:simple:PAC->HP}
  Let $k\in\NN_+$, let $\Omega$ and $\Lambda$ be non-empty Borel spaces with $\Lambda$ finite, let
  $\cH\subseteq\cF_k(\Omega,\Lambda)$ be a $k$-ary hypothesis class and let
  $\ell\colon\Omega^k\times\Lambda^{S_k}\times\Lambda^{S_k}\to\RR_{\geq 0}$ be a $k$-ary loss function.

  Let also
  \begin{align*}
    \gamma_\cH(m)
    & \df
    \sup_{x\in\Omega^m}
    \Bigl\lvert\Bigl\{
    \bigl(H(x_{\alpha(1)},\ldots,x_{\alpha(k)})\bigr)_{\alpha\in([m])_k}
    \;\Bigm\vert\;
    H\in\cH
    \Bigr\}\Bigr\rvert
  \end{align*}
  be the maximum number of different patterns (in $\Lambda^{([m])_k}$) that can be obtained by considering a fixed
  $x\in\Omega^m$ and plugging in each injective $k$-tuple as an input of some $H\in\cH$.

  If $\ell$ is metric and $\cH$ is $k$-PAC learnable with a $k$-PAC learner $\cA$, then $\cH$ has the Haussler packing property
  with associated function
  \begin{equation}\label{eq:simple:PAC->HP:mHP}
    m^{\HP}_{\cH,\ell}(\epsilon)
    \df
    \min_{\delta\in(0,1)}
    \Ceil{%
      \frac{%
        \gamma_\cH(\ceil{m^{\PAC}_{\cH,\ell,\cA}(\epsilon/2,\delta)})
      }{1-\delta}
    }
    - 2
    \leq
    \min_{\delta\in(0,1)}
    \Ceil{%
      \frac{%
        \lvert\Lambda\rvert^{(\ceil{m^{\PAC}_{\cH,\ell,\cA}(\epsilon/2,\delta)})_k}
      }{1-\delta}
    }
    - 2,
  \end{equation}
  where $(m)_k\df m(m-1)\cdots(m-k+1)$ denotes the falling factorial.
\end{theorem}

\begin{proof}
  First note that the inequality in~\eqref{eq:simple:PAC->HP:mHP} follows from the trivial bound
  \begin{equation*}
    \gamma_\cH(m) \leq \lvert\Lambda\rvert^{(m)_k}.
  \end{equation*}
  
  Second, note that due to the ceilings on the expressions in~\eqref{eq:simple:PAC->HP:mHP}, the minima are indeed attained
  as the functions only take values in $\NN$.

  Suppose for a contradiction that the result does not hold, that is, there exists $\epsilon\in(0,1)$ and $\mu\in\Pr(\Omega)$
  such that if $m$ is given by the first minimum in~\eqref{eq:simple:PAC->HP:mHP}, then for every $\cH'\subseteq\cH$ with
  $\lvert\cH'\rvert\leq m$, there exists $F\in\cH$ such that $L_{\mu,F,\ell}(H) > \epsilon$ for every $H\in\cH'$. By repeatedly
  applying this property, it follows that there exist $F_1,\ldots,F_{m+1}\in\cH$ such that for every $i,j\in[m+1]$ distinct, we
  have $L_{\mu,F_i,\ell}(F_j) > \epsilon$ (recall that $\ell$ is metric, so $L_{\mu,F_i,\ell}(F_j) = L_{\mu,F_j,\ell}(F_i)$).

  Let $\delta\in(0,1)$ attain the first minimum in~\eqref{eq:simple:PAC->HP:mHP} and let
  \begin{equation*}
    \widetilde{m}
    \df
    \Ceil{m^{\PAC}_{\cH,\ell,\cA}\left(\frac{\epsilon}{2},\delta\right)}.
  \end{equation*}

  For each $x\in\Omega^{\widetilde{m}}$, let
  \begin{equation*}
    Y(x)
    \df
    \Bigl\{
    \bigl(H(x_{\alpha(1)},\ldots,x_{\alpha(k)})\bigr)_{\alpha\in([\widetilde{m}])_k}
    \;\Bigm\vert\;
    H\in\cH
    \Bigr\}
    \subseteq \Lambda^{([\widetilde{m}])_k}
  \end{equation*}
  and note that $\lvert Y(x)\rvert\leq\gamma_\cH(\widetilde{m})$.

  For each $i\in[m+1]$, define the set
  \begin{equation*}
    C_i
    \df
    \left\{x\in\Omega^{\widetilde{m}} \;\middle\vert\;
    \forall y\in Y(x),
    L_{\mu,F_i,\ell}\bigl(\cA(x,y)\bigr)
    >
    \frac{\epsilon}{2}
    \right\}.
  \end{equation*}

  Note that by taking
  \begin{equation*}
    y \df \bigl(F_i(x_{\alpha(1)},\ldots,x_{\alpha(k)})\bigr)_{\alpha\in([m])_k}\in Y(x)
  \end{equation*}
  and using the fact that $L_{\mu,F_i,\ell}(F_i)=0$ (as $\ell$ is metric) so that $F_i$ is realizable in $\cH$ w.r.t.\ $\ell$
  and $\mu$, PAC learnability implies that $\mu(C_i)\leq\delta$.

  Define now the function $G\colon\Omega^{\widetilde{m}}\to\RR_{\geq 0}$ by
  \begin{align*}
    G(x)
    & \df
    \sum_{i=1}^{m+1} \One_{C_i}(x)
    \\
    & =
    \left\lvert\left\{i\in[m+1] \;\middle\vert\;
    \forall y\in Y(x),
    L_{\mu,F_i,\ell}\bigl(\cA(x,y)\bigr)
    >
    \frac{\epsilon}{2}
    \right\}\right\rvert.
  \end{align*}

  We claim that for every $x\in\Omega^{\widetilde{m}}$ and every $y\in Y(x)$, there exists at most one $i\in[m+1]$ such that
  $L_{\mu,F_i,\ell}(\cA(x,y))\leq\epsilon/2$. Indeed, if not, then for some $i,j\in[m+1]$ distinct, we would get
  \begin{equation*}
    L_{\mu,F_i,\ell}(F_j)
    \leq
    L_{\mu,F_i,\ell}\bigl(\cA(x,y)\bigr) + L_{\mu,F_j,\ell}\bigl(\cA(x,y)\bigr)
    \leq
    \epsilon,    
  \end{equation*}
  where the first inequality follows since $\ell$ is metric; the above would then contradict $L_{\mu,F_i,\ell}(F_j) > \epsilon$.

  Thus, we conclude that for every $x\in\Omega^{\widetilde{m}}$, we have
  \begin{equation}\label{eq:simple:Glowerbound}
    G(x)
    \geq
    m+1 - \lvert Y(x)\rvert
    \geq
    m+1 - \gamma_\cH(\widetilde{m}).
  \end{equation}

  On the other hand, since $\mu(C_i)\leq\delta$ for every $i\in[m+1]$, we get
  \begin{equation*}
    \int_{\Omega^{\widetilde{m}}} G(x)\ d\mu^{\widetilde{m}}(x)
    \leq
    (m+1)\delta,
  \end{equation*}
  which together with~\eqref{eq:simple:Glowerbound} implies
  \begin{equation*}
    m
    \leq
    \frac{\gamma_\cH(\widetilde{m})}{1-\delta} - 1,
  \end{equation*}
  contradicting the definitions of $m$, $\delta$ and $\widetilde{m}$.
\end{proof}

For the next theorem, we will need the following standard combinatorial result about covers of the power set of $[n]$.

\begin{lemma}\label{lem:coverbound}
  Let $c\in(0,1/2)$, let $n\in\NN$ and let $\cC\subseteq 2^{[n]}$ be a collection of subsets of $[n]$. Suppose that for every
  $U\subseteq[n]$, there exists $V\in\cC$ such that $\lvert U\symdiff V\rvert\leq c n$. Then
  \begin{equation*}
    n \leq \frac{\log_2 \lvert\cC\rvert}{1 - h_2(c)},
  \end{equation*}
  where
  \begin{equation*}
    h_2(t) \df t\log_2\frac{1}{t} + (1-t)\log_2\frac{1}{1-t}
  \end{equation*}
  denotes the binary entropy.
\end{lemma}

\begin{proof}
  For each $V\in\cC$, let
  \begin{equation*}
    B(V) \df \{U\subseteq [n] \mid \lvert U\symdiff V\rvert\leq cn\}
  \end{equation*}
  and note that
  \begin{equation*}
    \lvert B(V)\rvert = \sum_{i=0}^{\floor{cn}} \binom{n}{i} \leq 2^{h_2(c)\cdot n}, 
  \end{equation*}
  where the last inequality is the standard upper bound on the volume of a Hamming ball in terms of binary entropy (see
  e.g.~\cite[Lemma~4.7.2]{Ash65}).

  Since $\bigcup_{V\in\cC} B(V) = 2^{[n]}$, we conclude that
  \begin{equation*}
    \lvert\cC\rvert\cdot 2^{h_2(c)\cdot n} \geq 2^n,
  \end{equation*}
  from which the result follows.
\end{proof}

\begin{theorem}[Haussler packing property implies finite $\VCN_k$-dimension, simplified]\label{thm:simple:HP->VCNk}
  Let $k\in\NN_+$, let $\Omega$ and $\Lambda$ be non-empty Borel spaces with $\Lambda$ finite, let
  $\cH\subseteq\cF_k(\Omega,\Lambda)$ be a $k$-ary hypothesis class and let
  $\ell\colon\Omega^k\times\Lambda^{S_k}\times\Lambda^{S_k}\to\RR_{\geq 0}$ be a $k$-ary loss function.

  If $\ell$ is separated and $\cH$ has the Haussler packing property, then
  \begin{equation}\label{eq:simple:HP->VCNk:bound}
    \VCN_k(\cH)
    \leq
    \min_{\epsilon\in(0,\min\{s(\ell)\cdot k!/(2k^k), 1\})}
    \Floor{\frac{\log_2\floor{m^{\HP}_{\cH,\ell}(\epsilon)}}{1 - h_2(\epsilon\cdot k^k/(s(\ell)\cdot k!))}},
  \end{equation}
  where
  \begin{equation*}
    h_2(t) \df t\log_2\frac{1}{t} + (1-t)\log_2\frac{1}{1-t}
  \end{equation*}
  denotes the binary entropy.
\end{theorem}

\begin{proof}
  First, note that the minimum in~\eqref{eq:simple:HP->VCNk:bound} is indeed attained as the function only takes values in
  $\NN\cup\{-\infty\}$, so let $\epsilon\in(0,\min\{s(\ell)\cdot k!/(2k^k),1\})$ attain the minimum, let $d$ be the value of the
  minimum and let
  \begin{equation*}
    m
    \df
    \floor{m^{\HP}_{\cH,\ell}(\epsilon)}
  \end{equation*}
  so that
  \begin{equation*}
    d = \Floor{\frac{\log_2 m}{1-h_2(\epsilon\cdot k^k/(s(\ell)\cdot k!))}}.
  \end{equation*}

  When $\cH$ is empty, the result is trivial as $\VCN_k(\cH)=-\infty$, so suppose $\cH$ is non-empty (hence $m\geq 1$ and $d\geq
  0$).

  By the definition of $\VCN_k$-dimension, we have to show that if $x\in\Omega^{k-1}$, then $\Nat(\cH(x))\leq d$. In turn, it
  suffices to show that if $V\subseteq\Omega$ is a (finite) set that is Natarajan-shattered by $\cH(x)$ and $n\df\lvert
  V\rvert$, then $n\leq d$.

  Let $\mu\in\Pr(\Omega)$ be given by
  \begin{equation*}
    \frac{1}{k}\left(\nu_V + \sum_{j=1}^{k-1} \delta_{x_j}\right),
  \end{equation*}
  where $\nu_V$ is the uniform probability measure on $V$ and $\delta_t$ is the Dirac delta concentrated on $t$.

  Since $V$ is Natarajan-shattered by $\cH(x)$, there exist functions $f_0,f_1\colon V\to\Lambda^{S_k}$ with $f_0(v)\neq f_1(v)$
  for every $v\in V$ and there exists a family $\{F^U \mid U\subseteq V\}\subseteq\cH$ such that for every $U\subseteq V$ and
  every $v\in V$, we have $F^U_x(v) = f_{\One[v\in U]}(v)$.

  For $F,F'\in\cH$, let us define the set
  \begin{equation*}
    D(F,F') \df \{v\in V \mid F_x(v)\neq F'_x(v)\}.
  \end{equation*}
  Clearly, for every $U,U'\subseteq V$, we have $D(F_U,F_{U'})= U\symdiff U'$.
  
  Note also that by the definition of $\mu$, for $F,F'\in\cH$, we have
  \begin{equation}\label{eq:LFF'D}
    \begin{aligned}
      L_{\mu,F,\ell}(F')
      & \geq
      \PP_{\rn{z}\sim\mu^k}\bigl[
        \exists\sigma\in S_k, \forall j\in[k-1],
        \rn{z}_{\sigma(j)} = x_j\land \rn{z}_{\sigma(k)}\in D(F,F')
        \bigr]\cdot s(\ell)
      \\
      & \geq
      \frac{s(\ell)\cdot k!}{k^k\cdot n}\cdot\lvert D(F,F')\rvert.
    \end{aligned}
  \end{equation}

  Since $m$ is defined via Haussler packing property, we know that there exists $\cH'\subseteq\cH$ such that
  $\lvert\cH'\rvert\leq m$ and for every $U\subseteq V$, there exists $H\in\cH'$ such that $L_{\mu,F_U,\ell}(H)\leq\epsilon$.

  For each $H\in\cH'$, let
  \begin{align*}
    U_H & \df \{v\in V \mid H_x(v) = f_1(v)\},
    \\
    B(H) & \df \{U\subseteq V \mid L_{\mu,F_U,\ell}(H)\leq\epsilon\},
    \\
    B'(H) & \df
    \left\{U\subseteq V \;\middle\vert\;
    \lvert U\symdiff U_H\rvert\leq\frac{\epsilon\cdot k^k\cdot n}{s(\ell)\cdot k!}
    \right\}.
  \end{align*}

  The Haussler packing property assumption implies
  \begin{equation}\label{eq:simple:cover}
    \bigcup_{H\in\cH'} B(H) = 2^V.
  \end{equation}

  On the other hand, by~\eqref{eq:LFF'D}, for every $U\subseteq V$ and every $H\in\cH'$, we have
  \begin{equation*}
    L_{\mu,F_U,\ell}(H)
    \geq
    \frac{s(\ell)\cdot k!}{k^k\cdot n}\cdot\lvert D(F_U,H)\rvert
    \geq
    \frac{s(\ell)\cdot k!}{k^k\cdot n}\cdot\lvert D(F_U,F_{U_H})\rvert
    =
    \frac{s(\ell)\cdot k!}{k^k\cdot n}\cdot\lvert U\symdiff U_H\rvert,
  \end{equation*}
  which along with~\eqref{eq:simple:cover} implies $\bigcup_{H\in\cH'} B'(H) = 2^V$.

  Since $\epsilon\cdot k^k/(s(\ell)\cdot k!) < 1/2$, by Lemma~\ref{lem:coverbound}, we get
  \begin{equation*}
    n
    \leq
    \frac{\log_2\lvert\cH'\rvert}{1-h_2(\epsilon\cdot k^k\cdot n/(s(\ell)\cdot k!))}
    \leq
    \frac{\log_2 m}{1-h_2(\epsilon\cdot k^k\cdot n/(s(\ell)\cdot k!))},
  \end{equation*}
  which yields $n\leq d$ as $n$ is an integer and $d$ is the floor of the right-hand side of the above.
\end{proof}

\bibliographystyle{alpha}
\bibliography{refs}

\clearpage
\appendix

\section{High-arity PAC with high-order variables}
\label{sec:higharityPACdefs}

\afterpage{\begingroup
\def\bigbend{45}
\def\smallbend{20}
\def\tinybend{15}
\def\bendshift{5}
\begin{landscape}
  \vfill
  \begin{figure}[p]
    \centering
    \begin{small}
      \begin{tikzcd}
        \begin{tabular}{c}
          Non-partite finite\\
          $\VCN_k$-dimension
        \end{tabular}
        \arrow[d, twoheadleftarrow, two heads, dashed, "{\scriptsize\shortstack{\cite[8.3]{CM24+}}}"]
        &[1.0cm] &[0.8cm] 
        \begin{tabular}{c}
          Non-partite\\
          agnostically\\
          $k$-PAC learnable
        \end{tabular}
        \arrow[r, two heads, dashed, "{\scriptsize\shortstack{\cite[6.3(v)]{CM24+}}}"]
        \arrow[d, "{\scriptsize\shortstack{\cite[7.3(ii), 8.16]{CM24+}{}\\$\ell$ flexible, symmetric\\ and bounded}}", shift left]
        &[0.8cm]
        \begin{tabular}{c}
          Non-partite\\
          $k$-PAC learnable
        \end{tabular}
        \arrow[r, "{\scriptsize\shortstack{\ref{thm:simple:PAC->HP}, \ref{thm:PAC->HP}}}",
          "{\scriptsize\shortstack{$\ell$ (almost)\\metric\\and $\Lambda$ finite}}"']
        &[0.8cm]
        \begin{tabular}{c}
          Non-partite\\
          Haussler\\
          packing
        \end{tabular}
        \arrow[llll, tail, "{\scriptsize\shortstack{\ref{thm:simple:HP->VCNk}, \ref{thm:HP->VCNk}}}",
          "{\scriptsize\shortstack{$\rk(\cH)\leq 1$ and $\ell$ separated}}"',
          bend right={\tinybend}]
        \\[1.5cm]
        \begin{tabular}{c}
          Partite finite\\
          $\VCN_k$-dimension
        \end{tabular}
        \arrow[r, tail, "{\scriptsize\shortstack{\cite[9.12]{CM24+}}}"',
          "{\scriptsize\shortstack{$\ell$ local\\and bounded\\and $\Lambda$ finite}}"]
        &
        \begin{tabular}{c}
          Partite\\
          uniform\\
          convergence
        \end{tabular}
        \arrow[r, "{\scriptsize\shortstack{\cite[9.2]{CM24+}}}"',
          "{\scriptsize\shortstack{Existence of (almost)\\empirical risk\\minimizers}}"]
        & 
        \begin{tabular}{c}
          Partite\\
          agnostically\\
          $k$-PAC learnable
        \end{tabular}
        \arrow[r, two heads, dashed, "{\scriptsize\shortstack{\cite[6.3(v)]{CM24+}}}"]
        \arrow[u, two heads, dashed, shift left, "{\scriptsize\shortstack{\cite[8.4(ii)]{CM24+}}}"]
        &
        \begin{tabular}{c}
          Partite\\
          $k$-PAC learnable
        \end{tabular}
        \arrow[u, two heads, dashed, "{\scriptsize\shortstack{\cite[8.4(i)]{CM24+}}}"']
        \arrow[lll, "{\scriptsize\shortstack{\cite[10.2]{CM24+}}}"', tail,
          "{\scriptsize \shortstack{$\rk(\cH)\leq 1$ and $\ell$ separated}}",
          bend left={\tinybend}]
        \arrow[r, "{\scriptsize\shortstack{\ref{thm:PAC->HP}}}"',
          "{\scriptsize\shortstack{$\ell$ (almost)\\metric\\and $\Lambda$ finite}}"]
        &
        \begin{tabular}{c}
          Partite\\
          Haussler\\
          packing
        \end{tabular}
        \arrow[u, two heads, dashed, "{\scriptsize\shortstack{\ref{thm:HPpart->HP}}}"]
        \arrow[llll, tail, "{\scriptsize\shortstack{\ref{thm:HP->VCNk:partite}}}"',
          "{\scriptsize\shortstack{$\rk(\cH)\leq 1$ and $\ell$ separated}}",
          bend left={\smallbend}, shift left={\bendshift}]
      \end{tikzcd}
      \captionof{figure}{Diagram of implications between different high-arity PAC learning notions. Labels on arrows contain the
        number of the theorem (or the specific proposition in~\cite{CM24+}) that contains the proof of the implication and extra
        hypotheses needed. In the above ``$\ell$ (almost) metric'' means that either $\ell$ is metric or $\ell$ is separated and
        bounded. Arrows with two heads ($\twoheadrightarrow$) are tight in some sense with an obvious proof of tightness. Dashed
        arrows involve a construction (meaning that either the hypothesis class changes and/or the loss function changes) due to
        being in different settings; this also means that objects in one of the sides of the implication might not be completely
        general (as they are required to be in the image of the construction). Arrows with tails ($\rightarrowtail$) mean that
        exactly one of the sides involves a loss function. Under appropriate hypotheses, all items are proved equivalent (for
        example, if $\Lambda$ is finite and the loss is the $0/1$-loss $\ell_{0/1}$).}
      \label{fig:roadmap}
    \end{small}
  \end{figure}
  \vfill
\end{landscape}
\endgroup

}

In this section we lay out the definitions from high-arity PAC learning~\cite{CM24+} in their full generality including
high-order variables and the new definitions, theorems and proofs from this paper in the same full generality. The definitions
from~\cite{CM24+} are done in a streamlined manner, so we refer the reader to the original paper for a more thorough treatment
accompanied by intuition; to facilitate the process, numbers in square brackets below refer to the exact location of the
corresponding definition in~\cite{CM24+}.

\subsection{Definitions in the non-partite}

\begin{definition}[Borel templates~{[3.1]}]
  By a Borel space, we mean a standard Borel space, i.e., a measurable space that is Borel-isomorphic to a Polish space when
  equipped with the $\sigma$-algebra of Borel sets. The space of probability measures on a Borel space $\Lambda$ is denoted
  $\Pr(\Lambda)$.
  \begin{enumdef}
  \item{} [3.1.1] A \emph{Borel template} is a sequence $\Omega=(\Omega_i)_{i\in\NN_+}$, where $\Omega_i=(X_i,\cB_i)$ is a
    non-empty Borel space.
  \item{} [3.1.2] A \emph{probability template} on a Borel template $\Omega$ is a sequence $\mu = (\mu_i)_{i\in\NN_+}$, where
    $\mu_i\in\Pr(\Omega_i)$ is a probability measure on $\Omega_i$. We denote the space of probability templates on a Borel
    template $\Omega$ by $\Pr(\Omega)$.
  \item{} [3.1.4] For a (finite or) countable set $V$ and a Borel template $\Omega$, we define
    \begin{equation*}
      \cE_V(\Omega) \df \prod_{A\in r(V)} X_{\lvert A\rvert}
    \end{equation*}
    equipping it with the product $\sigma$-algebra, where
    \begin{equation}\label{eq:rV}
      r(V) \df \{A\subseteq V \mid A\text{ finite non-empty}\}.
    \end{equation}
    If $\mu\in\Pr(\Omega)$ is a probability template on $\Omega$, then we let $\mu^V\df\bigotimes_{A\in r(V)}\mu_{\lvert
      A\rvert}$ be the product measure. We use the shorthands $r(m)\df r([m])$, $\cE_m(\Omega)\df\cE_{[m]}(\Omega)$ and
    $\mu^m\df\mu^{[m]}$ when $m\in\NN$ (where $[m]\df\{1,\ldots,m\}$).
  \item{} [3.1.5] For an injective function $\alpha\colon U\to V$ between countable sets, we contra-variantly define the map
    $\alpha^*\colon\cE_V(\Omega)\to\cE_U(\Omega)$ by
    \begin{equation*}
      \alpha^*(x)_A \df x_{\alpha(A)} \qquad \bigl(x\in\cE_V(\Omega), A\in r(U)\bigr).
    \end{equation*}
  \end{enumdef}
\end{definition}

\begin{definition}[Hypotheses~{[3.2, 3.5]}]
  Let $\Omega$ be a Borel template, $\Lambda=(Y,\cB')$ be a non-empty Borel space and $k\in\NN_+$.
  \begin{enumdef}
  \item{} [3.2.1] The set of \emph{$k$-ary hypotheses} from $\Omega$ to $\Lambda$, denoted $\cF_k(\Omega,\Lambda)$, is the set
    of (Borel) measurable functions from $\cE_k(\Omega)$ to $\Lambda$.
  \item{} [3.2.2] A \emph{$k$-ary hypothesis class} is a subset $\cH$ of $\cF_k(\Omega,\Lambda)$ equipped with a
    $\sigma$-algebra such that:
    \begin{enumerate}
    \item the evaluation map $\ev\colon\cH\times\cE_k(\Omega)\to\Lambda$ given by $\ev(H,x)\df H(x)$ is measurable;
    \item for every $H\in\cH$, the set $\{H\}$ is measurable;
    \item for every Borel space $\Upsilon$ and every measurable set $A\subseteq\cH\times\Upsilon$, the projection of $A$ onto
      $\Upsilon$, i.e., the set
      \begin{equation*}
        \{\upsilon\in\Upsilon \mid \exists H\in\cH, (H,\upsilon)\in A\},
      \end{equation*}
      is universally measurable\footnote{Footnote~\ref{ftn:univmeas} also applies here.} (i.e., measurable in every
      completion of a probability measure on $\Upsilon$).
    \end{enumerate}
  \item{} [3.2.3] Given $F\in\cF_k(\Omega,\Lambda)$ and $m\in\NN$, we define the function
    $F^*_m\colon\cE_m(\Omega)\to\Lambda^{([m])_k}$ by
    \begin{equation*}
      F^*_m(x)_\alpha \df F(\alpha^*(x)) \qquad \bigl(x\in\cE_m(\Omega), \alpha\in([m])_k\bigr),
    \end{equation*}
    where $([m])_k$ is the set of injections $[k]\to[m]$; when $k=m$, we have $F^*_k\colon\cE_k(\Omega)\to\Lambda^{S_k}$, where
    $S_k=([k])_k$ is the symmetric group on $[k]$.
  \item{} [3.5.1] The \emph{rank} of a $k$-ary hypothesis $F\in\cF_k(\Omega,\Lambda)$, denoted $\rk(F)$ is the minimum $r\in\NN$
    such that $F$ factors as
    \begin{equation*}
      F(x) \df F'\bigl((x_A)_{A\in r(k),\lvert A\rvert\leq r}\bigr) \qquad \bigl(x\in\cE_k(\Omega)\bigr)
    \end{equation*}
    for some function $F'\colon\prod_{A\in r(k), \lvert A\rvert\leq r} X_A\to\Lambda$.
  \item{} [3.5.2] The \emph{rank} of a $k$-ary hypothesis class $\cH\subseteq\cF_k(\Omega,\Lambda)$ is defined as
    \begin{equation*}
      \rk(\cH) \df \sup_{F\in\cH} \rk(F).
    \end{equation*}
  \end{enumdef}
\end{definition}

\begin{definition}[Loss functions~{[3.7]}]
  Let $\Omega$ be a Borel template, $\Lambda$ be a non-empty Borel space and $k\in\NN_+$.
  \begin{enumdef}
  \item{} [3.7.1] A \emph{$k$-ary loss function} over $\Lambda$ is a measurable function
    $\ell\colon\cE_k(\Omega)\times\Lambda^{S_k}\times\Lambda^{S_k}\to\RR_{\geq 0}$.
  \item{} [3.7.2] For a $k$-ary loss function $\ell$, we define
    \begin{align*}
      \lVert\ell\rVert_\infty & \df \sup_{\substack{x\in\cE_k(\Omega)\\y,y'\in\Lambda^{S_k}}} \ell(x,y,y'),
      &
      s(\ell) & \df \inf_{\substack{x\in\cE_k(\Omega)\\y,y'\in\Lambda^{S_k}\\y\neq y'}} \ell(x,y,y').
    \end{align*}
  \item{} [3.7.3 simplified] A $k$-ary loss function is:
    \begin{description}[format={\normalfont\textit}]
    \item[bounded] if $\lVert\ell\rVert_\infty < \infty$.
    \item[separated] if $s(\ell) > 0$ and $\ell(x,y,y)=0$ for every $x\in\cE_k(\Omega)$ and every $y\in\Lambda^{S_k}$.
    \end{description}
  \item{} [3.7.4] For a $k$-ary loss function $\ell$, hypotheses $F,H\in\cF_k(\Omega,\Lambda)$ and a probability template
    $\mu\in\Pr(\Omega)$, the \emph{total loss} of $H$ with respect to $\mu$, $F$ and $\ell$ is
    \begin{equation*}
      L_{\mu,F,\ell}(H)
      \df
      \EE_{\rn{x}\sim\mu^k}\Bigl[\ell\bigl(\rn{x}, H^*_k(\rn{x}), F^*_k(\rn{x})\bigr)\Bigr].
    \end{equation*}
  \item{} [3.7.5] We say that $F\in\cF_k(\Omega,\Lambda)$ is \emph{realizable} in $\cH\subseteq\cF_k(\Omega,\Lambda)$ with
    respect to a $k$-ary loss function $\ell$ and $\mu\in\Pr(\Omega)$ if $\inf_{H\in\cH} L_{\mu,F,\ell}(H) = 0$.
  \end{enumdef}
\end{definition}

\begin{definition}[$k$-PAC learnability~{[3.8]}]
  Let $\Omega$ be a Borel template, $\Lambda$ be a non-empty Borel space and $\cH\subseteq\cF_k(\Omega,\Lambda)$ be a $k$-ary
  hypothesis class.
  \begin{enumdef}
  \item{} [3.8.2] A \emph{($k$-ary) learning algorithm} for $\cH$ is a measurable function
    \begin{equation*}
      \cA\colon\bigcup_{m\in\NN}\bigl(\cE_m(\Omega)\times\Lambda^{([m])_k}\bigr)\to\cH.
    \end{equation*}
  \item{} [3.8.3] We say that $\cH$ is \emph{$k$-PAC learnable} with respect to a $k$-ary loss function $\ell$ if there exist a
    learning algorithm $\cA$ for $\cH$ and a function $m^{\PAC}_{\cH,\ell,\cA}\colon(0,1)^2\to\RR_{\geq 0}$ such that for every
    $\epsilon,\delta\in(0,1)$, every $\mu\in\Pr(\Omega)$ and every $F\in\cF_k(\Omega,\Lambda)$ that is realizable in $\cH$ with
    respect to $\ell$ and $\mu$, we have
    \begin{equation*}
      \PP_{\rn{x}\sim\mu^m}\biggl[
        L_{\mu,F,\ell}\Bigl(\cA\bigl(\rn{x}, F^*_m(\rn{x})\bigr)\Bigr)
        \leq \epsilon
        \biggr] \geq 1 - \delta
    \end{equation*}
    for every integer $m\geq m^{\PAC}_{\cH,\ell,\cA}(\epsilon,\delta)$. A learning algorithm $\cA$ satisfying the above is
    called a \emph{$k$-PAC learner} for $\cH$ with respect to $\ell$.
  \end{enumdef}
\end{definition}


\begin{definition}[$\VCN_k$-dimension~{[3.14]}]
  Let $\Omega$ be a Borel template, $\Lambda$ be a non-empty Borel space, $k\in\NN_+$ and $\cH\subseteq\cF_k(\Omega,\Lambda)$ be
  a $k$-ary hypothesis class.
  \begin{enumdef}
  \item{} [3.14.1] For $H\in\cF_k(\Omega,\Lambda)$ and $x\in\cE_{k-1}(\Omega)$, let
    \begin{equation*}
      H^*_k(x,\place)\colon\prod_{A\in r(k)\setminus r(k-1)} X_{\lvert A\rvert}\to\Lambda^{S_k}
    \end{equation*}
    be the function obtained from $H^*_k$ by fixing its $\cE_{k-1}(\Omega)$ arguments to be $x$ and let
    \begin{equation*}
      \cH(x)\df\{H^*_k(x,\place) \mid H\in\cH\}.
    \end{equation*}
  \item{} [3.14.2] The \emph{Vapnik--Chervonenkis--Natarajan $k$-dimension} of $\cH$ (\emph{$\VCN_k$-dimension}) is defined as
    \begin{equation*}
      \VCN_k(\cH) \df \sup_{x\in\cE_{k-1}(\Omega)} \Nat\bigl(\cH(x)\bigr).
    \end{equation*}
  \end{enumdef}
\end{definition}


\subsection{Definitions in the partite}

\begin{definition}[Borel $k$-partite templates~{[4.1]}]
  Let $k\in\NN_+$.
  \begin{enumdef}
  \item{} [4.1.1] A \emph{Borel $k$-partite template} is a sequence $\Omega=(\Omega_A)_{A\in r(k)}$, where
    $\Omega_A=(X_A,\cB_A)$ is a non-empty (standard) Borel space and $r(k)=r([k])$ is given by~\eqref{eq:rV}.
  \item{} [4.1.2] A \emph{probability $k$-partite template} on a Borel $k$-partite template $\Omega$ is a sequence
    $\mu=(\mu_A)_{A\in r(k)}$, where $\mu_A$ is a probability measure on $\Omega_A$. The space of probability $k$-partite
    templates on $\Omega$ is denoted $\Pr(\Omega)$.
  \item{} [4.1.4 simplified] For a Borel $k$-partite template $\Omega$, a non-empty Borel space $\Lambda$ and $m\in\NN_+$, we
    define
    \begin{equation*}
      \cE_m(\Omega) \df \prod_{f\in r_k(m)} X_{\dom(f)},
    \end{equation*}
    equipping it with the product $\sigma$-algebra, where
    \begin{equation*}
      r_k(m) \df \{f\colon A\to [m] \mid A\in r(k)\} = \bigcup_{A\in r(k)} [m]^A.
    \end{equation*}
    If $\mu\in\Pr(\Omega)$ is a probability $k$-partite template on $\Omega$, we let $\mu^m\df\bigotimes_{f\in r_k(m)}
    \mu_{\dom(f)}$ be the product measure.
  \item{} [4.1.5 simplified] For $\alpha\in [m]^k$, we define the map $\alpha^*\colon\cE_m(\Omega)\to\cE_1(\Omega)$ by
    \begin{equation*}
      \alpha^*(x)_f \df x_{\alpha\rest_{\dom(f)}} \qquad \bigl(x\in\cE_m(\Omega), f\in r_k(1)\bigr).
    \end{equation*}
  \end{enumdef}
\end{definition}

\begin{definition}[$k$-Partite hypotheses~{[4.2, 4.5]}]
  Let $k\in\NN_+$, let $\Omega$ be a Borel $k$-partite template and let $\Lambda=(Y,\cB')$ be a non-empty Borel space.
  \begin{enumdef}
  \item{} [4.2.1] The set of \emph{$k$-partite hypotheses} from $\Omega$ to $\Lambda$, denoted $\cF_k(\Omega,\Lambda)$, is the
    set f (Borel) measurable functions from $\cE_1(\Omega)$ to $\Lambda$.
  \item{} [4.2.2] A \emph{$k$-partite hypothesis class} is a subset $\cH$ of $\cF_k(\Omega,\Lambda)$ equipped with a
    $\sigma$-algebra such that:
    \begin{enumerate}
    \item the evaluation map $\ev\colon\cH\times\cE_1(\Omega)\to\Lambda$ given by $\ev(H,x)\df H(x)$ is measurable;
    \item for every $H\in\cH$, the set $\{H\}$ is measurable;
    \item for every Borel space $\Upsilon$ and every measurable set $A\subseteq\cH\times\Upsilon$, the projection of $A$ onto
      $\Upsilon$, i.e., the set
      \begin{equation*}
        \{\upsilon\in\Upsilon \mid \exists H\in\cH, (H,\upsilon)\in A\},
      \end{equation*}
      is universally measurable\footnote{Footnote~\ref{ftn:univmeas} also applies here.} (i.e., measurable in every completion
      of a probability measure on $\Upsilon$).
    \end{enumerate}
  \item{} [4.2.3 simplified] For a $k$-partite hypothesis $F\in\cF_k(\Omega,\Lambda)$, we let
    $F^*_m\colon\cE_m(\Omega)\to\Lambda^{[m]^k}$ be given by
    \begin{equation*}
      F^*_m(x)_\alpha \df F(\alpha^*(x)) \qquad \bigl(x\in\cE_m(\Omega), \alpha\in[m]^k\bigr).
    \end{equation*}
  \item{} [4.5.1] The \emph{rank} of a $k$-partite hypothesis $F\in\cF_k(\Omega,\Lambda)$, denote $\rk(F)$ is the minimum
    $r\in\NN$ such that $F$ factors as
    \begin{equation*}
      F(x) = F'\bigl((x_f)_{f\in r_k(1), \lvert\dom(f)\rvert\leq r}\bigr) \qquad \bigl(x\in\cE_1(\Omega)\bigr)
    \end{equation*}
    for some function $F'\colon\prod_{f\in r_k(1),\lvert\dom(f)\rvert\leq r} X_{\dom(f)}\to\Lambda$.
  \item{} [4.5.2] The \emph{rank} of a $k$-partite hypothesis class $\cH\subseteq\cF_k(\Omega,\Lambda)$ is defined as
    \begin{equation*}
      \rk(\cH) \df \sup_{F\in\cH} \rk(F).
    \end{equation*}
  \end{enumdef}
\end{definition}

\begin{definition}[$k$-Partite loss functions~{[4.7]}]
  Let $k\in\NN_+$, let $\Omega$ be a Borel $k$-partite template and let $\Lambda$ be a non-empty Borel space.
  \begin{enumdef}
  \item{} [4.7.1] A \emph{$k$-partite loss function} over $\Lambda$ is a measurable function
    $\ell\colon\cE_1(\Omega)\times\Lambda\times\Lambda\to\RR_{\geq 0}$.
  \item{} [4.7.2] For a $k$-partite loss function $\ell$, we define
    \begin{align*}
      \lVert\ell\rVert_\infty & \df \sup_{\substack{x\in\cE_1(\Omega)\\y,y'\in\Lambda}} \ell(x,y,y'),
      &
      s(\ell) & \df \inf_{\substack{x\in\cE_1(\Omega)\\y,y'\in\Lambda\\y\neq y'}} \ell(x,y,y').
    \end{align*}
  \item{} [4.7.3] A $k$-partite loss function is:
    \begin{description}[format={\normalfont\textit}]
    \item[bounded] if $\lVert\ell\rVert_\infty < \infty$.
    \item[separated] if $s(\ell) > 0$ and $\ell(x,y,y)$ for every $x\in\cE_1(\Omega)$ and every $y\in\Lambda$.
    \end{description}
  \item{} [4.7.4] For a $k$-partite loss function $\ell$, $k$-partite hypotheses $F,H\in\cF_k(\Omega,\Lambda)$ and a probability
    $k$-partite template $\mu\in\Pr(\Omega)$, the \emph{total loss} of $H$ with respect to $\mu$, $F$ and $\ell$ is
    \begin{equation*}
      L_{\mu,F,\ell}(H)
      \df
      \EE_{\rn{x}\sim\mu^1}\Bigl[\ell\bigl(\rn{x}, H(\rn{x}), F(\rn{x})\bigr)\Bigr].
    \end{equation*}
  \item{} [4.7.5] We say that $F\in\cF_k(\Omega,\Lambda)$ is \emph{realizable} in $\cH\subseteq\cF_k(\Omega,\Lambda)$ with
    respect to a $k$-partite loss function $\ell$ and $\mu\in\Pr(\Omega)$ if $\inf_{H\in\cH} L_{\mu,F,\ell}(H) = 0$.
  \item{} [4.7.6] The \emph{$k$-partite $0/1$-loss function} over $\Lambda$ is defined as $\ell_{0/1}(x,y,y')\df\One[y\neq y']$.
  \end{enumdef}
\end{definition}

\begin{definition}[Partite $k$-PAC learnability~{[4.8]}]
  Let $k\in\NN_+$, let $\Omega$ be a Borel $k$-partite template, let $\Lambda$ be a non-empty Borel space and let
  $\cH\subseteq\cF_k(\Omega,\Lambda)$ be a $k$-partite hypothesis class.
  \begin{enumdef}
  \item{} [4.8.2] A \emph{($k$-partite) learning algorithm} for $\cH$ is a measurable function
    \begin{equation*}
      \cA\colon\bigcup_{m\in\NN}\bigl(\cE_m(\Omega)\times\Lambda^{[m]^k}\bigr)\to\cH.
    \end{equation*}
  \item{} [4.8.3] We say that $\cH$ is \emph{$k$-PAC learnable} with respect to a $k$-partite loss function $\ell$ if there
    exist a learning algorithm $\cA$ for $\cH$ and a function $m^{\PAC}_{\cH,\ell,\cA}\colon(0,1)^2\to\RR_{\geq 0}$ such that
    for every $\epsilon,\delta\in(0,1)$, every $\mu\in\Pr(\Omega)$ and every $F\in\cF_k(\Omega,\Lambda)$ that is realizable in
    $\cH$ with respect to $\ell$ and $\mu$, we have
    \begin{equation*}
      \PP_{\rn{x}\sim\mu^m}\biggl[
        L_{\mu,F,\ell}\Bigl(\cA\bigl(\rn{x}, F^*_m(\rn{x})\bigr)\Bigr)
        \leq \epsilon
        \biggr] \geq 1 - \delta
    \end{equation*}
    for every integer $m\geq m^{\PAC}_{\cH,\ell,\cA}(\epsilon,\delta)$. A learning algorithm $\cA$ satisfying the above is
    called a \emph{$k$-PAC learner} for $\cH$ with respect to $\ell$.
  \end{enumdef}
\end{definition}


\begin{definition}[Partite $\VCN_k$-dimension~{[4.13]}]
  Let $k\in\NN_+$, let $\Omega$ be a Borel $k$-partite template, let $\Lambda$ be a non-empty Borel space and let
  $\cH\subseteq\cF_k(\Omega,\Lambda)$ be a $k$-partite hypothesis class.
  \begin{enumdef}
  \item{} [4.13.1] For $A\in\binom{[k]}{k-1}$, let
    \begin{equation*}
      r_{k,A} \df \{f\in r_k(1) \mid \dom(f)\subseteq A\},
    \end{equation*}
    for $x\in\prod_{f\in r_{k,A}} X_{\dom(f)}$ and $H\in\cF_k(\Omega,\Lambda)$, let
    \begin{equation*}
      H(x,\place)\colon\prod_{f\in r_k(1)\setminus r_{k,A}} X_{\dom(f)}\to\Lambda
    \end{equation*}
    be the function obtained from $H$ by fixing its arguments in $\prod_{f\in r_{k,A}} X_{\dom(f)}$ to be $x$ and let
    \begin{equation*}
      \cH(x) \df \{H(x,\place) \mid H\in\cH\}.
    \end{equation*}
  \item{} [4.13.2] The \emph{Vapnik--Chervonenkis--Natarajan $k$-dimension} of $\cH$ (\emph{$\VCN_k$-dimension}) is defined as
    \begin{equation*}
      \VCN_k(\cH) \df \sup_{\substack{A\in\binom{[k]}{k-1}\\ x\in\prod_{f\in r_{k,A}} X_{\dom(f)}}} \Nat\bigl(\cH(x)\bigr).
    \end{equation*}
  \end{enumdef}
\end{definition}


\subsection{New high-arity PAC definitions}
\label{subsec:newhigharityPACdefs}

In this subsection, we lay out the new high-arity definitions of this paper in full generality.

\begin{definition}[$k$-ary Haussler packing property]
  Let $k\in\NN_+$, let $\Omega$ be a Borel ($k$-partite, respectively) template, let $\Lambda$ be a non-empty Borel space, let
  $\cH\subseteq\cF_k(\Omega,\Lambda)$ be a $k$-ary ($k$-partite, respectively) hypothesis class and let
  $\ell\colon\cE_k(\Omega)\times\Lambda^{S_k}\times\Lambda^{S_k}\to\RR_{\geq 0}$
  ($\ell\colon\cE_1(\Omega)\times\Lambda\times\Lambda\to\RR_{\geq 0}$, respectively) be a $k$-ary ($k$-partite, respectively)
  loss function.

  We say that $\cH$ has the \emph{$k$-ary Haussler packing property} with respect to $\ell$ if there exists a function
  $m^{\HP}_{\cH,\ell}\colon(0,1)\to\RR_{\geq 0}$ such that for every $\epsilon\in(0,1)$ and every $\mu\in\Pr(\Omega)$, there
  exists $\cH'\subseteq\cH$ with $\lvert\cH'\rvert\leq m^{\HP}_{\cH,\ell}(\epsilon)$ such that for every $F\in\cH$, there exists
  $H\in\cH'$ such that $L_{\mu,F,\ell}(H)\leq\epsilon$. We refer to elements of $\cH'$ as \emph{$k$-ary Haussler centers} of
  $\cH$ at precision $\epsilon$ with respect to $\mu$ and $\ell$.
\end{definition}

\begin{definition}[Metric loss functions]
  Let $k\in\NN_+$, let $\Omega$ be a Borel ($k$-partite, respectively) template, let $\Lambda$ be a non-empty Borel space.

  We say that a $k$-ary ($k$-partite, respectively) loss function
  $\ell\colon\cE_k(\Omega)\times\Lambda^{S_k}\times\Lambda^{S_k}\to\RR_{\geq 0}$
  ($\ell\colon\cE_1(\Omega)\times\Lambda\times\Lambda\to\RR_{\geq 0}$, respectively) is \emph{metric} if for every
  $x\in\cE_k(\Omega)$ ($x\in\cE_1(\Omega)$, respectively), the function $\ell(x,\place,\place)$ is a metric on $\Lambda^{S_k}$
  ($\Lambda$, respectively) in the usual sense, that is, the following hold for every $x\in\cE_k(\Omega)$ and
  $y,y',y''\in\Lambda^{S_k}$ ($x\in\cE_1(\Omega)$ and $y,y',y''\in\Lambda$, respectively):
  \begin{enumerate}
  \item We have $\ell(x,y,y')=\ell(x,y',y)$.
  \item We have $\ell(x,y,y')=0$ if and only if $y = y'$.
  \item We have $\ell(x,y,y'')\leq \ell(x,y,y') + \ell(x,y',y'')$.
  \end{enumerate}
\end{definition}

\subsection{Main results}

In this section, we prove the main results in full generality including the high-order variables. For the particular case of the
counterpart of Theorem~\ref{thm:simple:PAC->HP}, Theorem~\ref{thm:PAC->HP}, we will also show that instead of assuming that the
loss function $\ell$ is metric, one could assume that $\ell$ is separated and bounded; for this, we will use the lemma below
that says that separated and bounded loss functions make the total loss satisfy a weak version of triangle inequality with a
rescaling factor. This justifies the usage of the name ``(almost) metric'' for losses that are either metric or both separated
and bounded in Figure~\ref{fig:roadmap}. Finally, we point out that the $0/1$-loss function $\ell_{0/1}$ satisfies all
hypotheses of Theorems~\ref{thm:PAC->HP}, \ref{thm:HP->VCNk:partite} and~\ref{thm:HP->VCNk}.

\begin{lemma}\label{lem:almostmetric}
  Let $\Omega$ be a Borel ($k$-partite, respectively) template, let $\Lambda$ be a non-empty Borel space, let
  $\cH\subseteq\cF_k(\Omega,\Lambda)$ be a $k$-ary ($k$-partite, respectively) hypothesis class, let $\ell$ be a $k$-ary
  ($k$-partite, respectively) loss function and let $\mu\in\Pr(\Omega)$ be a probability ($k$-partite, respectively) template.

  For each $F,H\in\cH$, let
  \begin{equation*}
    D(F,H) \df
    \begin{dcases*}
      \{x\in\cE_k(\Omega) \mid F^*_k(x)\neq H^*_k(x)\}, & in the non-partite case,\\
      \{x\in\cE_1(\Omega) \mid F(x)\neq H(x)\}, & in the partite case.
    \end{dcases*}
  \end{equation*}
  Then the following hold:
  \begin{enumerate}
  \item\label{lem:almostmetric:bounds} We have
    \begin{equation*}
      s(\ell)\cdot M(F,H) \leq L_{F,\mu,\ell}(H) \leq \lVert\ell\rVert_\infty\cdot M(F,H),
    \end{equation*}
    where
    \begin{equation*}
      M(F,H) \df
      \begin{dcases*}
        \mu^k(D(F,H)), & in the non-partite case,\\
        \mu^1(D(F,H)), & in the partite case.
      \end{dcases*}
    \end{equation*}
  \item\label{lem:almostmetric:triangle} If $\ell$ is separated and $F,F',H\in\cH$, then
    \begin{equation*}
      L_{\mu,F,\ell}(F')
      \leq
      \frac{\lVert\ell\rVert_\infty}{s(\ell)}\bigl(L_{\mu,F,\ell}(H) + L_{\mu,F',\ell}(H)\bigr)
    \end{equation*}
  \end{enumerate}
\end{lemma}

\begin{proof}
  Item~\ref{lem:almostmetric:bounds} follows since
  \begin{equation*}
    L_{\mu,F,\ell}(H) =
    \begin{dcases*}
      \EE_{\rn{x}\sim\mu^k}\Bigl[\ell\bigl(\rn{x},H^*_k(\rn{x}),F^*_k(\rn{x})\bigr)\Bigr], & in the non-partite case,\\
      \EE_{\rn{x}\sim\mu^1}\ell\Bigl[\bigl(\rn{x},H(\rn{x}),F(\rn{x})\bigr)\Bigr], & in the partite case.
    \end{dcases*}
  \end{equation*}

  \medskip

  For item~\ref{lem:almostmetric:triangle}, by item~\ref{lem:almostmetric:bounds}, we have
  \begin{align*}
    L_{\mu,F,\ell}(F')
    & \leq
    \lVert\ell\rVert_\infty\cdot M(F,F')
    \leq
    \lVert\ell\rVert_\infty\cdot\bigl(M(F,H) + M(F',H)\bigr)
    \\
    & \leq
    \frac{\lVert\ell\rVert_\infty}{s(\ell)}\cdot\bigl(L_{\mu,F,\ell}(H) + L_{\mu,F',\ell}(H)\bigr)
  \end{align*}
  where the second inequality is the (usual) triangle inequality.
\end{proof}

We now prove the counterpart of Theorem~\ref{thm:simple:PAC->HP} both in the non-partite and partite settings.

\begin{theorem}[$k$-PAC learnability implies Haussler packing property]\label{thm:PAC->HP}
  Let $k\in\NN_+$, let $\Omega$ be a Borel ($k$-partite, respectively) template, let $\Lambda$ be a finite non-empty Borel
  space, let $\cH\subseteq\cF_k(\Omega,\Lambda)$ be a $k$-ary ($k$-partite, respectively) hypothesis class and let $\ell$ be a
  $k$-ary ($k$-partite, respectively) loss function. Suppose that $\cH$ is $k$-PAC learnable with a $k$-PAC learner $\cA$.

  Let also
  \begin{equation*}
    \gamma_\cH(m)
    \df
    \sup_{x\in\cE_m(\Omega)} \lvert\{H^*_m(x) \mid H\in\cH\}\rvert
  \end{equation*}
  be the maximum number of different patterns (in $\Lambda^{([m])_k}$ in the non-partite case or in $\Lambda^{[m]^k}$ in the
  partite case) that can be obtained by evaluating all elements of $\cH$ in a fixed $x\in\cE_m(\Omega)$. Then the following
  hold:
  \begin{enumerate}
  \item\label{thm:PAC->HP:sepbounded} If $\ell$ is separated and bounded, then $\cH$ has the Haussler
    packing property with associated function
    \begin{equation}\label{eq:PAC->HP:mHP}
      \begin{aligned}
        m^{\HP}_{\cH,\ell}(\epsilon)
        & \df
        \min_{\delta\in(0,1)}
        \Ceil{%
          \frac{%
            \gamma_\cH(\ceil{m^{\PAC}_{\cH,\ell,\cA}(s(\ell)\epsilon/(2\lVert\ell\rVert_\infty),\delta)})
          }{1-\delta}
        }
        - 2
        \\
        & \leq
        \begin{dcases*}
          \min_{\delta\in(0,1)}
          \Ceil{%
            \frac{%
              \lvert\Lambda\rvert^{(\ceil{m^{\PAC}_{\cH,\ell,\cA}(s(\ell)\epsilon/(2\lVert\ell\rVert_\infty),\delta)})_k}
            }{1-\delta}
          }
          - 2,
          & in the non-partite case,
          \\
          \min_{\delta\in(0,1)}
          \Ceil{%
            \frac{%
              \lvert\Lambda\rvert^{\ceil{m^{\PAC}_{\cH,\ell,\cA}(s(\ell)\epsilon/(2\lVert\ell\rVert_\infty),\delta)}^k}
            }{1-\delta}}
          - 2,
          & in the partite case.
        \end{dcases*}
      \end{aligned}
    \end{equation}
  \item\label{thm:PAC->HP:metric} If $\ell$ is metric, then $\cH$ has the Haussler packing property
    with associated function
    \begin{equation}\label{eq:PAC->HP:mHP:metric}
      \begin{aligned}
        m^{\HP}_{\cH,\ell}(\epsilon)
        & \df
        \min_{\delta\in(0,1)}
        \Ceil{%
          \frac{%
            \gamma_\cH(\ceil{m^{\PAC}_{\cH,\ell,\cA}(\epsilon/2,\delta)})
          }{1-\delta}
        }
        - 2
        \\
        & \leq
        \begin{dcases*}
          \min_{\delta\in(0,1)}
          \Ceil{%
            \frac{%
              \lvert\Lambda\rvert^{(\ceil{m^{\PAC}_{\cH,\ell,\cA}(\epsilon/2,\delta)})_k}
            }{1-\delta}
          }
          - 2,
          & in the non-partite case,
          \\
          \min_{\delta\in(0,1)}
          \Ceil{%
            \frac{%
              \lvert\Lambda\rvert^{\ceil{m^{\PAC}_{\cH,\ell,\cA}(\epsilon/2,\delta)}^k}
            }{1-\delta}}
          - 2,
          & in the partite case.
        \end{dcases*}
      \end{aligned}
    \end{equation}
  \end{enumerate}
\end{theorem}

\begin{proof}
  For item~\ref{thm:PAC->HP:sepbounded}, note that due to the ceilings on the right-hand sides of~\eqref{eq:PAC->HP:mHP}, the
  minima are indeed attained as the functions only take values in $\NN$.

  The inequalities also clearly follow from the trivial bound:
  \begin{equation*}
    \gamma_\cH(m)
    \leq
    \begin{dcases*}
      \lvert\Lambda\rvert^{(m)_k}, & in the non-partite case,\\
      \lvert\Lambda\rvert^{m^k}, & in the partite case.
    \end{dcases*}
  \end{equation*}
  
  Suppose for a contradiction that the result does not hold, that is, there exist $\epsilon\in(0,1)$ and $\mu\in\Pr(\Omega)$
  such that if $m$ is given by the right-hand side of~\eqref{eq:PAC->HP:mHP}, then for every $\cH'\subseteq\cH$ with
  $\lvert\cH'\rvert\leq m$, there exists $F\in\cH$ such that $L_{\mu,F,\ell}(H) > \epsilon$ for every $H\in\cH'$. By starting
  with $\cH'\df\varnothing$ and inductively applying this property with $\cH'\df\{F_1,\ldots,F_t\}$ to produce $F_{t+1}$, it
  follows that there exist $F_1,\ldots,F_{m+1}\in\cH$ such that for every $i,j\in[m+1]$ with $i < j$, we have
  $L_{\mu,F_i,\ell}(F_j) > \epsilon$.

  Let $\delta\in(0,1)$ attain the first minimum in~\eqref{eq:PAC->HP:mHP} and let
  \begin{equation*}
    \widetilde{m}
    \df
    \Ceil{m^{\PAC}_{\cH,\ell,\cA}\left(\frac{s(\ell)\cdot\epsilon}{2\cdot\lVert\ell\rVert_\infty},\delta\right)}.
  \end{equation*}

  For each $x\in\cE_{\widetilde{m}}(\Omega)$, let
  \begin{equation*}
    Y(x) \df \{H^*_m(x) \mid H\in\cH\}
  \end{equation*}
  and note that $\lvert Y(x)\rvert\leq\gamma_\cH(\widetilde{m})$.

  For each $i\in[m+1]$, define the set
  \begin{equation*}
    C_i
    \df
    \left\{x\in\cE_{\widetilde{m}}(\Omega) \;\middle\vert\;
    \forall y\in Y(x),
    L_{\mu,F_i,\ell}\bigl(\cA(x,y)\bigr) > \frac{s(\ell)\cdot\epsilon}{2\cdot\lVert\ell\rVert_\infty}
    \right\}.
  \end{equation*}

  Note that by taking $y\df (F_i)^*_m(x)\in Y(x)$ and using the fact that $\ell$ is separated so that $F_i$ is realizable in
  $\cH$ w.r.t.\ $\ell$ and $\mu$, PAC learnability implies that $\mu(C_i)\leq\delta$.
  
  Define now the function $G\colon\cE_{\widetilde{m}}(\Omega)\to\RR_{\geq 0}$ by
  \begin{align*}
    G(x)
    & \df
    \sum_{i=1}^{m+1} \One_{C_i}(x)
    \\
    & =
    \left\lvert\left\{i\in[m+1] \;\middle\vert\;
    \forall y\in Y(x),
    L_{\mu,F_i,\ell}\bigl(\cA(x,y)\bigr) > \frac{s(\ell)\cdot\epsilon}{2\cdot\lVert\ell\rVert_\infty}
    \right\}\right\rvert.
  \end{align*}

  We claim that for every $x\in\cE_{\widetilde{m}}(\Omega)$ and every $y\in Y(x)$, there exists at
  most one $i\in[m+1]$ such that $L_{\mu,F_i,\ell}(\cA(x,y))\leq
  s(\ell)\epsilon/(2\lVert\ell\rVert_\infty)$. Indeed, if not, then for some $i,j\in[m+1]$ with $i <
  j$, we would get
  \begin{equation*}
    \frac{s(\ell)\cdot\epsilon}{\lVert\ell\rVert_\infty}
    \geq
    L_{\mu,F_i,\ell}\bigl(\cA(x,y)\bigr) + L_{\mu,F_j,\ell}\bigl(\cA(x,y)\bigr)
    \geq
    \frac{s(\ell)}{\lVert\ell\rVert_\infty}\cdot L_{\mu,F_i,\ell}(F_j),
  \end{equation*}
  where the last inequality follows from
  Lemma~\ref{lem:almostmetric}\ref{lem:almostmetric:triangle}; the above would then contradict
  $L_{\mu,F_i,\ell}(F_j) > \epsilon$.

  Thus, we conclude that
  \begin{equation}\label{eq:Glowerbound}
    G(x)
    \geq
    m+1 - \lvert Y(x)\rvert
    \geq
    m+1 - \gamma_\cH(\widetilde{m})
  \end{equation}
  for every $x\in\cE_{\widetilde{m}}(\Omega)$.

  On the other hand, since $\mu(C_i)\leq\delta$ for every $i\in[m+1]$, we get
  \begin{equation*}
    \int_{\cE_{\widetilde{m}}(\Omega)} G(x)\ d\mu^{\widetilde{m}}(x)
    \leq
    (m+1)\delta,
  \end{equation*}
  which together with~\eqref{eq:Glowerbound} implies
  \begin{equation*}
    m
    \leq
    \frac{\gamma_\cH(\widetilde{m})}{1-\delta} - 1,
  \end{equation*}
  contradicting the definitions of $m$, $\delta$ and $\widetilde{m}$.

  \medskip

  The proof of item~\ref{thm:PAC->HP:metric} is completely analogous (see the proof of Theorem~\ref{thm:simple:PAC->HP}), but
  instead of using Lemma~\ref{lem:almostmetric}, we use triangle inequality since $\ell$ is metric, which allows us to improve
  the bound to~\eqref{eq:PAC->HP:mHP:metric} instead of~\eqref{eq:PAC->HP:mHP} (also note that the fact that $\ell$ is metric
  implies that $L_{\mu,F_i,\ell}(F_i)=0$).
\end{proof}

\begin{remark}\label{rmk:bootstrap}
  The final bound provided on Theorem~\ref{thm:PAC->HP} is provably not tight when $\rk(\cH)\leq 1$ and $\ell$ is bounded. This
  is because a posteriori, we know that all results of high-arity PAC in~\cite{CM24+} along with Theorem~\ref{thm:HP->VCNk}
  prove that the notions considered are also equivalent to finiteness of $\VCN_k$-dimension, which in turn implies that
  \begin{equation}\label{eq:VCNkfullgrowth}
    \begin{aligned}
      \gamma_\cH(m)
      & \leq
      \begin{dcases*}
        (m+1)^{\VCN_k(\cH)\cdot\binom{m}{k-1}}
        \cdot\binom{\lvert\Lambda\rvert}{2}^{\VCN_k(\cH)\cdot\binom{m}{k-1}},
        & in the non-partite case,
        \\
        (m+1)^{\VCN_k(\cH)\cdot m^{k-1}}
        \cdot\binom{\lvert\Lambda\rvert}{2}^{\VCN_k(\cH)\cdot m^{k-1}},
        & in the partite case.
      \end{dcases*}
      \\
      & \leq
      \left(\frac{\lvert\Lambda\rvert^2\cdot (m+1)}{2}\right)^{\VCN_k(\cH)\cdot m^{k-1}},
    \end{aligned}
  \end{equation}
  instead of the trivial bounds for $\gamma_\cH(m)$ used on Theorem~\ref{thm:PAC->HP}.
\end{remark}

For the counterpart of Theorem~\ref{thm:simple:HP->VCNk}, it will be more convenient to split it into the partite version
(Theorem~\ref{thm:HP->VCNk:partite}) and the non-partite version (Theorem~\ref{thm:HP->VCNk}). We start with the easier one: the
partite.

\begin{theorem}[Haussler packing property implies finite $\VCN_k$-dimension, partite version]\label{thm:HP->VCNk:partite}
  Let $k\in\NN_+$, let $\Omega$ be a Borel $k$-partite template, let $\Lambda$ be a finite non-empty Borel space, let
  $\cH\subseteq\cF_k(\Omega,\Lambda)$ be a $k$-partite hypothesis class and let
  $\ell\colon\cE_1(\Omega)\times\Lambda\times\Lambda\to\RR_{\geq 0}$ be a $k$-partite loss function. Suppose $\ell$ is separated
  and $\rk(\cH)\leq 1$.

  If $\cH$ has the Haussler packing property, then
  \begin{equation}\label{eq:HP->VCNk:partite:bound}
    \VCN_k(\cH)
    \leq
    \min_{\epsilon\in(0,\min\{s(\ell)/2,1\})}
    \Floor{\frac{\log_2\floor{m^{\HP}_{\cH,\ell}(\epsilon)}}{1 - h_2(\epsilon/s(\ell))}},
  \end{equation}
  where
  \begin{equation*}
    h_2(t) \df t\log_2\frac{1}{t} + (1-t)\log_2\frac{1}{1-t}
  \end{equation*}
  denotes the binary entropy.
\end{theorem}

\begin{proof}
  Note that the minimum in~\eqref{eq:HP->VCNk:partite:bound} is indeed attained as the function only takes values in
  $\NN\cup\{-\infty\}$, so let $\epsilon\in(0,\min\{s(\ell)/2,1\})$ attain the minimum, let $d$ be the value of the minimum and
  let
  \begin{equation*}
    m
    \df
    \floor{m^{\HP}_{\cH,\ell}(\epsilon)}
  \end{equation*}
  so that
  \begin{equation*}
    d = \Floor{\frac{\log_2 m}{1-h_2(\epsilon/s(\ell))}}.
  \end{equation*}

  When $\cH$ is empty, the result is trivial as $\VCN_k(\cH)=-\infty$ so suppose $\cH$ is non-empty (hence $m\geq 1$ and $d\geq
  0$).

  By the definition of $\VCN_k$-dimension, we have to show that if $A\in\binom{[k]}{k-1}$ and $x\in\prod_{f\in r_{k,A}}
  X_{\dom(f)}$, then $\Nat(\cH(x))\leq d$. In turn, it suffices to show that if $V\subseteq\prod_{f\in r_k(1)\setminus r_{k,A}}
  X_{\dom(f)}$ is a (finite) set that is Natarajan-shattered by $\cH(x)$, then $\lvert V\rvert\leq d$.

  Let
  \begin{equation*}
    R \df r_k(1)\setminus (r_{k,A}\cup\{1^{\{a\}}\})
  \end{equation*}
  where $1^{\{a\}}$ is the unique function $\{a\}\to[1]$ and let $x'\in\prod_{f\in R} X_{\dom(f)}$ be any fixed point.

  Since $\rk(\cH)\leq 1$, it follows that the projection $V'$ of $V$ onto the coordinate indexed by $1^{\{a\}}$ is
  Natarajan-shattered by
  \begin{equation*}
    \{H(x,x',\place) \mid H\in\cH\}.
  \end{equation*}

  Let $n\df\lvert V\rvert$ (which is equal to $\lvert V'\rvert$ as $\rk(\cH)\leq 1$ and $V$ is Natarajan-shattered by $\cH(x)$)
  and let $\mu\in\Pr(\Omega)$ be the probability $k$-partite template given by:
  \begin{itemize}
  \item For each $f\in r_{k,A}$, $\mu_f$ is the Dirac delta concentrated on $x_f$.
  \item For each $f\in R$, $\mu_f$ is the Dirac delta concentrated on $x'_f$.
  \item $\mu_{1^{\{a\}}}$ is the uniform measure on $V'$.
  \end{itemize}

  Since $V'$ is Natarajan-shattered by $\{H(x,x',\place) \mid H\in\cH\}$, there exist functions $f_0,f_1\colon V'\to\Lambda$
  with $f_0(v)\neq f_1(v)$ for every $v\in V'$ and there exists a family $\{F_U\mid U\subseteq V'\}\subseteq\cH$ such that for
  every $U\subseteq V'$ and every $v\in V'$, we have $F_U(x,x',v) = f_{\One[v\in U]}(v)$.

  Using now our definition of $m$ via Haussler packing property, we know that there exists $\cH'\subseteq\cH$ such that
  $\lvert\cH'\rvert\leq m$ and for every $U\subseteq V'$, there exists $H\in\cH'$ such that $L_{\mu,F_U,\ell}(H)\leq\epsilon$.

  For each $H\in\cH'$, let
  \begin{align*}
    U_H & \df \{v\in V' \mid H(x,x',v) = f_1(v)\},
    \\
    B(H) & \df \{U\subseteq V' \mid L_{\mu,F_U,\ell}(H)\leq\epsilon\},
    \\
    B'(H)
    & \df
    \left\{U\subseteq V' \;\middle\vert\;
    \lvert U\symdiff U_H\rvert\leq\frac{\epsilon\cdot n}{s(\ell)}
    \right\}.
  \end{align*}

  The Haussler packing property assumption implies
  \begin{equation}\label{eq:cover:partite}
    \bigcup_{H\in\cH'} B(H) = 2^{V'}.
  \end{equation}
  Since $\ell$ is separated, by
  Lemma~\ref{lem:almostmetric}\ref{lem:almostmetric:bounds}, for every $H\in\cH'$ and every
  $U\subseteq V'$, we have
  \begin{equation*}
    L_{\mu,F_U,\ell}(H)
    \geq
    s(\ell)\cdot\mu^k(D(F_U,H))
    \geq
    s(\ell)\cdot\mu^k(D(F_U,F_{U_H}))
    \geq
    \frac{s(\ell)}{n}\cdot\lvert U\symdiff U_H\rvert,
  \end{equation*}
  hence $B(H)\subseteq B'(H)$, which together with~\eqref{eq:cover:partite} implies
  $\bigcup_{H\in\cH'} B'(H) = 2^{V'}$.

  Since $\epsilon/s(\ell) < 1/2$, by Lemma~\ref{lem:coverbound}, we get
  \begin{equation*}
    n
    \leq
    \frac{\log_2 \lvert\cH'\rvert}{1 - h_2(\epsilon/s(\ell))}
    \leq
    \frac{\log_2 m}{1 - h_2(\epsilon/s(\ell))},
  \end{equation*}
  which yields $n\leq d$ as $n$ is an integer.
\end{proof}

\begin{theorem}[Haussler packing property implies finite $\VCN_k$-dimension, non-partite version]\label{thm:HP->VCNk}
  Let $k\in\NN_+$, let $\Omega$ be a Borel template, let $\Lambda$ be a finite non-empty Borel space, let
  $\cH\subseteq\cF_k(\Omega,\Lambda)$ be a $k$-ary hypothesis class and let
  $\ell\colon\cE_k(\Omega)\times\Lambda^{S_k}\times\Lambda^{S_k}\to\RR_{\geq 0}$ be a $k$-ary loss function. Suppose that $\ell$
  is separated and $\rk(\cH)\leq 1$.

  If $\cH$ has the Haussler packing property, then
  \begin{equation}\label{eq:HP->VCNk:bound}
    \VCN_k(\cH)
    \leq
    \min_{\epsilon\in(0,\min\{s(\ell)\cdot k!/(2k^k), 1\})}
    \Floor{\frac{\log_2\floor{m^{\HP}_{\cH,\ell}(\epsilon)}}{1 - h_2(\epsilon\cdot k^k/(s(\ell)\cdot k!))}},
  \end{equation}
  where
  \begin{equation*}
    h_2(t) \df t\log_2\frac{1}{t} + (1-t)\log_2\frac{1}{1-t}
  \end{equation*}
  denotes the binary entropy.
\end{theorem}

\begin{proof}
  Note that the minimum in~\eqref{eq:HP->VCNk:bound} is indeed attained as the function only takes values in
  $\NN\cup\{-\infty\}$, so let $\epsilon\in(0,\min\{s(\ell) k!/(2k^k),1\})$ attain the minimum, let $d$ be the value of the
  minimum and let
  \begin{equation*}
    m
    \df
    \floor{m^{\HP}_{\cH,\ell}(\epsilon)}
  \end{equation*}
  so that
  \begin{equation*}
    d = \Floor{\frac{\log_2 m}{1-h_2(\epsilon\cdot k^k/(s(\ell)\cdot k!))}}.
  \end{equation*}

  When $\cH$ is empty, the result is trivial as $\VCN_k(\cH)=-\infty$, so suppose $\cH$ is non-empty (hence $m\geq 1$ and $d\geq
  0$).

  By the definition of $\VCN_k$-dimension, we have to show that if $x\in\cE_{k-1}(\Omega)$, then $\Nat(\cH(x))\leq d$. In turn,
  it suffices to show that if $V\subseteq\prod_{A\in r(k)\setminus r(k-1)} X_{\lvert A\rvert}$ is a (finite) set that is
  Natarajan-shattered by $\cH(x)$, then $\lvert V\rvert\leq d$.

  Let
  \begin{equation*}
    R \df r(k)\setminus(r(k-1)\cup\{\{k\}\})
  \end{equation*}
  and let $x'\in\prod_{A\in R} X_A$ be any fixed point.

  Since $\rk(\cH)\leq 1$, it follows that the projection $V'$ of $V$ onto the coordinate indexed by $\{k\}$ is
  Natarajan-shattered by
  \begin{equation*}
    \{H^*_k(x,x',\place) \mid H\in\cH\}.
  \end{equation*}
  Note that $\lvert V'\rvert = \lvert V\rvert$ as $\rk(\cH)\leq 1$ and $V$ is Natarajan-shattered by $\cH(x)$; thus our goal is
  to show $\lvert V'\rvert\leq d$.

  Let $n\df\lvert V\rvert$ (which is equal to $\lvert V'\rvert$ as $\rk(\cH)\leq 1$ and $V$ is Natarajan-shattered by $\cH(x)$)
  and let $\mu\in\Pr(\Omega)$ be the probability template given by letting $\mu_i$ be any probability measure on $\Omega_i$ for
  each $i\in\NN$ with $i\geq 2$ and letting
  \begin{equation*}
    \mu_1
    \df
    \frac{1}{k}\left(\nu_{V'} + \sum_{j=1}^{k-1} \delta_{x_{\{j\}}}\right),
  \end{equation*}
  where $\nu_{V'}$ is the uniform probability measure on $V'$ and $\delta_t$ is the Dirac delta concentrated on $t$.

  Since $V'$ is Natarajan-shattered by $\{H^*_k(x,x',\place) \mid H\in\cH\}$, there exist functions $f_0,f_1\colon
  V'\to\Lambda^{S_k}$ with $f_0(v)\neq f_1(v)$ for every $v\in V'$ and there exists a family $\{F_U\mid U\subseteq
  V'\}\subseteq\cH$ such that for every $U\subseteq V'$ and every $v\in V'$, we have $(F_U)^*_k(x,x',v) = f_{\One[v\in U]}(v)$.

  Recall that for $H_1,H_2\in\cH$, Lemma~\ref{lem:almostmetric} defines $D(H_1,H_2)\df\{x\in\cE_k(\Omega)\mid (H_1)^*_k(x) \neq
  (H_2)^*_k(x)\}$. Let us also define
  \begin{equation*}
    D'(H_1,H_2)\df\{v\in V' \mid (H_1)^*_k(x,x',v)\neq (H_2)^*_k(x,x',v)\}.
  \end{equation*}
  Clearly, for every $U,U'\subseteq V'$, we have $D'(F_U,F_{U'})=U\symdiff U'$.

  Note that
  \begin{equation}\label{eq:muDD'}
    \begin{aligned}
    \mu^k(D(H_1,H_2))
    & \geq
    \PP_{\rn{z}\sim\mu^k}\bigl[
        \exists\sigma\in S_k, \forall j\in[k-1],
        \rn{z}_{\{\sigma(j)\}} = x_{\{j\}}\land \rn{z}_{\{\sigma(k)\}}\in D'(H_1,H_2)
        \bigr]\cdot s(\ell)
    \\
    & \geq
    \frac{k!}{k^k\cdot n}\cdot\lvert D'(H_1,H_2)\rvert.
    \end{aligned}
  \end{equation}
  (This is true even if there are repetitions among $x_{\{1\}},\ldots,x_{\{k-1\}}$ and even if $V'$ contains some of these
  elements.)

  Using now our definition of $m$ via Haussler packing property, we know that there exists $\cH'\subseteq\cH$ such that
  $\lvert\cH'\rvert\leq m$ and for every $U\subseteq V'$, there exists $H\in\cH'$ such that $L_{\mu,F_U,\ell}(H)\leq\epsilon$.

  For each $H\in\cH'$, let
  \begin{align*}
    U_H & \df \{v\in V'\mid H(x,x',v) = f_1(v)\},
    \\
    B(H) & \df \{U\subseteq V' \mid L_{\mu,F_U,\ell}(H)\leq\epsilon\},
    \\
    B'(H)
    & \df
    \left\{U\subseteq V' \;\middle\vert\;
    \lvert U\symdiff U_H\rvert\leq\frac{\epsilon\cdot k^k\cdot n}{s(\ell)\cdot k!}
    \right\}.
  \end{align*}

  The Haussler packing property assumption implies
  \begin{equation}\label{eq:cover}
    \bigcup_{H\in\cH'} B(H) = 2^{V'}.
  \end{equation}
  Since $\ell$ is separated, by Lemma~\ref{lem:almostmetric}\ref{lem:almostmetric:bounds}, for every
  $H\in\cH'$ and every $U\subseteq V'$, we have
  \begin{align*}
    L_{\mu,F_U,\ell}(H)
    & \geq
    s(\ell)\cdot\mu^k(D(F_U,H))
    \geq
    \frac{s(\ell)\cdot k!}{k^k\cdot n}\cdot\lvert D'(F_U,H)\rvert
    \\
    & \geq
    \frac{s(\ell)\cdot k!}{k^k\cdot n}\cdot\lvert D'(F_U,F_{U_H})\rvert
    =
    \frac{s(\ell)\cdot k!}{k^k\cdot n}\cdot\lvert U\symdiff U_H\rvert,
  \end{align*}
  where the second inequality follows from~\eqref{eq:muDD'}, hence $B(H)\subseteq B'(H)$, which
  together with~\eqref{eq:cover} implies $\bigcup_{H\in\cH'} B'(H) = 2^{V'}$.

  Since $\epsilon\cdot k^k/(s(\ell)\cdot k!) < 1/2$, by Lemma~\ref{lem:coverbound}, we get
  \begin{equation*}
    n
    \leq
    \frac{\log_2 \lvert\cH'\rvert}{1-h_2(\epsilon\cdot k^k\cdot n/(s(\ell)\cdot k!))}
    \leq
    \frac{\log_2 m}{1-h_2(\epsilon\cdot k^k\cdot n/(s(\ell)\cdot k!))},
  \end{equation*}
  which yields $n\leq d$ as $n$ is an integer.
\end{proof}

\section{The partization operation}
\label{sec:partization}

In this section, we recall the definition of the partization operation from~\cite[Definition~4.20]{CM24+} and we prove that the
partite Haussler packing property of $\cH^{\kpart}$ implies the non-partite Haussler packing property of $\cH$. Similarly to
Section~\ref{sec:higharityPACdefs}, numbers in square brackets in the definition below refer to the exact location of the
concept in~\cite{CM24+}.

\begin{definition}[Partization~{[4.20]}]
  Let $k\in\NN_+$, let $\Omega$ be a Borel template and let $\Lambda$ be a non-empty Borel space.
  \begin{enumdef}
  \item{} [4.20.1] The \emph{$k$-partite version} of $\Omega$ is the Borel $k$-partite template $\Omega^{\kpart}$ given by
    $\Omega^{\kpart}_A\df\Omega_{\lvert A\rvert}$ ($A\in r(k)$).
  \item{} [4.20.2] For $\mu\in\Pr(\Omega)$, the \emph{$k$-partite version} of $\mu$ is $\mu^{\kpart}\in\Pr(\Omega^{\kpart})$
    given by $\mu^{\kpart}_A\df\mu_{\lvert A\rvert}$ ($A\in r(k)$).
  \item{} [4.20.3] For a hypothesis $F\in\cF_k(\Omega,\Lambda)$, the \emph{$k$-partite version} of $F$ is the $k$-partite
    hypothesis $F^{\kpart}\in\cF_k(\Omega^{\kpart},\Lambda^{S_k})$ given by
    \begin{equation*}
      F^{\kpart}(x) \df F^*_k\bigl(\iota_{\kpart}(x)\bigr) \qquad \bigl(x\in\cE_1(\Omega^{\kpart})\bigr),
    \end{equation*}
    where $\iota_{\kpart}\colon\cE_1(\Omega^{\kpart})\to\cE_k(\Omega)$ is given by
    \begin{equation}\label{eq:iotakpart}
      \iota_{\kpart}(x)_A \df x_{1^A} \qquad \bigl(x\in\cE_1(\Omega^{\kpart}), A\in r(k)\bigr)
    \end{equation}
    and $1^A\in r_k(1)$ is the unique function $A\to[1]$.
  \item{} [4.20.4] For a hypothesis class $\cH\subseteq\cF_k(\Omega,\Lambda)$, the \emph{$k$-partite version} of $\cH$ is
    $\cH^{\kpart}\df\{H^{\kpart}\mid H\in\cH\}$, equipped with the pushforward $\sigma$-algebra of the one of
    $\cH$. In~\cite[Lemma~8.1]{CM24+} (see Lemma~\ref{lem:kpartbasics} below), it is shown that $\iota_{\kpart}$ is a
    Borel-isomorphism, which in turn implies that $\cH\ni F\mapsto F^{\kpart}\in\cH^{\kpart}$ is a bijection (so singletons of
    $\cH^{\kpart}$ are indeed measurable) and that $\cH\mapsto\cH^{\kpart}$ is an injection. We denote by $\cH^{\kpart}\ni
    G\mapsto G^{\kpart,-1}\in\cH$ the inverse of $\cH\ni F\mapsto F^{\kpart}\in\cH^{\kpart}$.
  \item{} [4.20.5] For a $k$-ary loss function $\ell$ over $\Lambda$, the \emph{$k$-partite version} of $\ell$ is
    $\ell^{\kpart}\colon\cE_1(\Omega^{\kpart})\times\Lambda^{S_k}\times\Lambda^{S_k}\to\RR_{\geq 0}$ given by
    \begin{equation*}
      \ell^{\kpart}(x,y,y') \df \ell\bigl(\iota_{kpart}(x),y,y')
      \qquad \bigl(\cE_1(\Omega^{\kpart}),y,y'\in\Lambda^{S_k}\bigr).
    \end{equation*}
  \end{enumdef}
\end{definition}

\begin{lemma}[Partization basics~\protect{\cite[Lemma~8.1]{CM24+}}]\label{lem:kpartbasics}
  Let $\Omega$ be a Borel template, let $k\in\NN_+$, let $\Lambda$ be a non-empty Borel space. Then the following hold:
  \begin{enumerate}
  \item\label{lem:kpartbasics:phi} For $\mu\in\Pr(\Omega)$ and $m\in\NN$ the function
    $\phi_m\colon\cE_m(\Omega)\to\cE_{\floor{m/k}}(\Omega^{\kpart})$ given by
    \begin{align}\label{eq:kpartbasics:phi}
      \phi_m(x)_f
      & \df
      x_{\{(i-1)\floor{m/k} + f(i) \mid i\in\dom(f)\}}
      \qquad \left(f\in r_k\left(\floor{m/k}\right)\right).
    \end{align}
    is measure-preserving with respect to $\mu^m$ and $(\mu^{\kpart})^{\floor{m/k}}$. Furthermore, if $m$ is divisible by $k$,
    then $\phi_m$ is a measure-isomorphism.

    Moreover, we have $\phi_k^{-1} = \iota_{\kpart}$, where $\iota_{\kpart}$ is given
    by~\eqref{eq:iotakpart}.
  \item\label{lem:kpartbasics:Phi} For $m\in\NN$, $F\in\cF_k(\Omega,\Lambda)$ and
    $\Phi_m\colon\Lambda^{([m])_k}\to(\Lambda^{S_k})^{[\floor{m/k}]^k}$ given by
    \begin{align}\label{eq:kpartbasics:Phi}
      (\Phi_m(y)_\alpha)_\tau
      & \df
      y_{\beta_\alpha\comp\tau}
      \qquad \left(\alpha\in \left[\floor{m/k}\right]^k, \tau\in S_k\right),
    \end{align}
    where $\beta_\alpha\in([m])_k$ is given by
    \begin{align}
      \beta_\alpha(i)
      & \df
      (i-1)\Floor{\frac{m}{k}} + \alpha(i)
      \qquad \left(\alpha\in \left[\Floor{\frac{m}{k}}\right]^k, i\in[k]\right),
    \end{align}
    the diagram
    \begin{equation*}
      \begin{tikzcd}[column sep={2.5cm}]
        \cE_m(\Omega)
        \arrow[r, "F^*_m"]
        \arrow[d, "\phi_m"']
        &
        \Lambda^{([m])_k}
        \arrow[d, "\Phi_m"]
        \\
        \cE_{\floor{m/k}}(\Omega^{\kpart})
        \arrow[r, "(F^{\kpart})^*_{\floor{m/k}}"]
        &
        (\Lambda^{S_k})^{[\floor{m/k}]^k}
      \end{tikzcd}
    \end{equation*}
    commutes, where $\phi_m$ is given by~\eqref{eq:kpartbasics:phi}.
  \end{enumerate}
\end{lemma}

The following lemma is equation~(8.7) within the proof of~\cite[Proposition~8.4]{CM24+}. For completeness purposes, we restate
this below with a self-contained proof.

\begin{lemma}\label{lem:kpartloss}
  Let $\Omega$ be a Borel template, let $k\in\NN_+$, let $\Lambda$ be a non-empty Borel space, let
  $\ell\colon\cE_k(\Omega)\times\Lambda^{S_k}\times\Lambda^{S_k}\to\RR_{\geq 0}$ be a $k$-ary loss function and let
  $F,H\in\cF_k(\Omega,\Lambda)$ be hypotheses. Then
  \begin{equation*}
    L_{\mu,F,\ell}(H) = L_{\mu^{\kpart},F^{\kpart},\ell^{\kpart}}(H^{\kpart}).
  \end{equation*}
\end{lemma}

\begin{proof}
  This follows directly from
  \begin{align*}
    L_{\mu,F,\ell}(H)
    & =
    \EE_{\rn{x}\sim\mu^k}[\ell(\rn{x},H^*_k(\rn{x}),F^*_k(\rn{x}))]
    \\
    & =
    \EE_{\rn{x}\sim\mu^k}[\ell^{\kpart}(\phi_k(\rn{x}),H^{\kpart}(\phi_k(\rn{x})),F^{\kpart}(\phi_k(\rn{x})))]
    \\
    & =
    \EE_{\rn{z}\sim(\mu^{\kpart})^1}[\ell^{\kpart}(\rn{z},H^{\kpart}(\rn{z}), F^{\kpart}(\rn{z}))]
    \\
    & =
    L_{\mu^{\kpart},F^{\kpart},\ell^{\kpart}}(H^{\kpart}),
  \end{align*}
  where the second equality follows from the definition of $\ell^{\kpart}$ and
  Lemma~\ref{lem:kpartbasics}\ref{lem:kpartbasics:Phi} and the third equality follows since $\phi_m$ is measure-preserving by
  Lemma~\ref{lem:kpartbasics}\ref{lem:kpartbasics:phi}.
\end{proof}

\begin{theorem}[Haussler property: partite to non-partite]\label{thm:HPpart->HP}
  Let $\Omega$ be a Borel template, let $k\in\NN_+$, let $\Lambda$ be a non-empty Borel space, let
  $\cH\subseteq\cF_k(\Omega,\Lambda)$ be a $k$-ary hypothesis class, let
  $\ell\colon\cE_k(\Omega)\times\Lambda^{S_k}\times\Lambda^{S_k}\to\RR_{\geq 0}$ be a $k$-ary loss function.

  If $\cH^{\kpart}$ has the Haussler packing property with respect to $\ell^{\kpart}$, then $\cH$ has the Haussler packing
  property with respect to $\ell$ with associated function $m^{\HP}_{\cH,\ell}\df m^{\HP}_{\cH^{\kpart},\ell^{\kpart}}$.
\end{theorem}

\begin{proof}
  For $\epsilon\in(0,1)$ and $\mu\in\Pr(\Omega)$, we know that there exists $\widetilde{\cH}\subseteq\cH^{\kpart}$ with
  $\lvert\widetilde{\cH}\rvert\leq m^{\HP}_{\cH^{\kpart},\ell^{\kpart}}(\epsilon)$ such that
  \begin{equation*}
    \forall F\in\cH^{\kpart}, \exists H\in\widetilde{\cH}, L_{\mu^{\kpart},F,\ell^{\kpart}}(H)\leq\epsilon
  \end{equation*}

  Letting $\cH'\df\widetilde{\cH}^{\kpart,-1}$, the above implies
  \begin{equation*}
    \forall F\in\cH, \exists H\in\cH', L_{\mu^{\kpart},F^{\kpart},\ell^{\kpart}}(H^{\kpart})\leq\epsilon,
  \end{equation*}
  which by Lemma~\ref{lem:kpartloss} yields
  \begin{equation*}
    \forall F\in\cH, \exists H\in\cH', L_{\mu,F,\ell}(H)\leq\epsilon,
  \end{equation*}
  as desired.
\end{proof}

\end{document}